\documentclass{ecai} 
\usepackage{times}
\usepackage{soul}
\usepackage{url}
\usepackage{graphicx}
\usepackage{amsmath}
\usepackage{amsthm}
\usepackage{booktabs}
\usepackage{algorithm}
\usepackage{algorithmic}
\usepackage{multirow}
\usepackage{graphicx}

\usepackage{nohyperref}

\usepackage{amsfonts}       % blackboard math symbols
\usepackage{nicefrac}       % compact symbols for 1/2, etc.
\usepackage{microtype}      % microtypography
\usepackage{xcolor}         % colors
\usepackage{tablefootnote}

\usepackage{tikz} % nice language for creating drawings and diagrams
\usetikzlibrary{arrows}
\usetikzlibrary{arrows.meta}
\usetikzlibrary{shapes.multipart}
\usetikzlibrary{chains}

\tikzset{
 anode/.style={rectangle, draw, minimum size = 12pt, inner sep = 4pt},
 snode/.style={rectangle split, rectangle split parts=2, draw, rectangle split horizontal, minimum size=18pt, inner sep=4pt, text=black},
>=Stealth}

\newtheorem{theorem}{Theorem}

\newcommand{\BibTeX}{B\kern-.05em{\sc i\kern-.025em b}\kern-.08em\TeX}

\newcommand{\leftt}{\textrm{left}}
\newcommand{\rightt}{\textrm{right}}
\newcommand{\weight}{\textrm{value}}

\begin{document}

%%%%%%%%%%%%%%%%%%%%%%%%%%%%%%%%%%%%%%%%%%%%%%%%%%%%%%%%%%%%%%%%%%%%%%%%

\begin{frontmatter}

%%% Use this command to specify your submission number.
%%% In doubleblind mode, it will be printed on the first page.

%\paperid{2807} 

%%% Use this command to specify the title of your paper.

\title{Vanilla Gradient Descent for Oblique Decision Trees}

%%% Use this combinations of commands to specify all authors of your 
%%% paper. Use \fnms{} and \snm{} to indicate everyone's first names 
%%% and surname. This will help the publisher with indexing the 
%%% proceedings. Please use a reasonable approximation in case your 
%%% name does not neatly split into "first names" and "surname".
%%% Specifying your ORCID digital identifier is optional. 
%%% Use the \thanks{} command to indicate one or more corresponding 
%%% authors and their email address(es). If so desired, you can specify
%%% author contributions using the \footnote{} command.

\author[A]{\fnms{Subrat Prasad}~\snm{Panda}\thanks{Corresponding Author. Email: subratpr001@e.ntu.edu.sg}}
\author[B]{\fnms{Blaise}~\snm{Genest}}
\author[A]{\fnms{Arvind}~\snm{Easwaran}}
\author[C]{\fnms{Ponnuthurai Nagaratnam}~\snm{Suganthan}} 

\address[A]{NTU, Singapore}
\address[B]{CNRS, IPAL, France and CNRS@CREATE, Singapore} 
\address[C]{KINDI Computing Research, Qatar University, Qatar}

%%% Use this environment to include an abstract of your paper.

\begin{abstract}
Decision Trees (DTs) constitute one of the major highly non-linear AI models, valued, e.g., for their efficiency on tabular data. 
Learning accurate DTs is, however, complicated, especially for oblique DTs, and does take a significant training time. Further, DTs suffer from overfitting, e.g., they proverbially "do not generalize" in regression tasks.
Recently, some works proposed ways to make (oblique) DTs differentiable.
This enables highly efficient gradient-descent algorithms to be used to learn DTs. It also enables generalizing capabilities by learning regressors at the leaves simultaneously with the decisions in the tree. Prior approaches to making DTs differentiable rely either on probabilistic approximations at the tree's internal nodes (soft DTs) or on approximations in gradient computation at the internal node (quantized gradient descent). In this work, we propose \textit{DTSemNet}, a novel \textit{sem}antically equivalent and invertible encoding for (hard, oblique) DTs as Neural \textit{Net}works (NNs), that uses standard vanilla gradient descent. Experiments across various classification and regression benchmarks show that oblique DTs learned using \textit{DTSemNet} are more accurate than oblique DTs of similar size learned using state-of-the-art techniques. Further, DT training time is significantly reduced. We also experimentally demonstrate that \textit{DTSemNet} can learn DT policies as efficiently as NN policies in the Reinforcement Learning (RL) setup with physical inputs (dimensions $\leq32$). The code is available at {\color{blue}\textit{\url{https://github.com/CPS-research-group/dtsemnet}}}.
\end{abstract}

\end{frontmatter}

%%%%%%%%%%%%%%%%%%%%%%%%%%%%%%%%%%%%%%%%%%%%%%%%%%%%%%%%%%%%%%%%%%%%%%%%
\section{Introduction}\label{sec:intro}

DTs are widely adopted in several domains, such as the medical domain \citep{chen2017machine}.
Studies, including recent ones \citep{grinsztajn2022tree}, show that DTs classify tabular datasets very well due to their inductive bias toward learning non-smooth functions, even better than NNs. 
Learning DTs for a given task is, however, a complex task: due to the combinatorial explosion over choices of the decisions to branch over at every node of the tree, learning optimal DTs is an NP-hard problem \citep{laurent1976constructing}. 

Current methods of learning DTs can be broadly classified into four approaches: (a) greedy optimization, which grows trees using split criteria such as Classification and Regression Trees (CART) \citep{breiman1984cart}, (b) non-greedy optimization, which jointly optimize decision nodes under global objective such as Tree Alternating Optimization (TAO) \citep{carreira2018alternating, tao-regression}), (c) Global searches, either presented as Mixed Integer Programs (MIP) \citep{bennett1996optimal}, or using Evolutionary Algorithms (EA), e.g. \citep{ea2023}, and (d) Gradient Descent by making the tree differentiable \citep{hazimeh2020tree,yang2018deep,zantedeschi2021learning}.
Greedy optimization techniques usually learn low-performing trees \citep{zantedeschi2021learning, dgt2022}, global searches are computationally expensive when the number of possible tree configurations is large \citep{bertsimas2017optimal, ea2023} and non-greedy approaches, although computationally better, are still computationally more expensive than gradient-based approaches \cite{dgt2022}. Moreover, training without gradient descent in the RL setting is very challenging, as there is no dataset: the agent learns from experience collected in an environment. A workaround is first to learn an NN using gradient descent, and imitate the NN (treated as an oracle) with a (classification-) DT using greedy strategies such as CART, which works reasonably well for simple RL problems with a small 
(up to $\approx 6$ features) {\em discrete} action space \cite{bastani2018verifiable}. On the other hand, recent results with the gradient-based approach have shown to be efficient in terms of training time and accuracy for both classification and regression tasks \citep{dgt2022}. Additionally, it is capable of directly learning (regression-) DT policies in RL settings with {\em continuous} actions \cite{icct2022}.

To apply gradient descent at the training stage on DTs, most previous works adopt ``soft'' decisions (e.g., using the \texttt{Sigmoid} activation function) to make the tree differentiable \cite{lee2019oblique,biau2019neural,zantedeschi2021learning,hazimeh2020tree}. However, the tree obtained does not provide ``hard'' decisions but ``soft'', that is, probabilistic ones. Hard DTs can be obtained from soft DTs, but accuracy decreases. Some recent studies, such as Dense Gradient Trees (DGT) \cite{dgt2022} and Interpretable Continous Control Trees (ICCT) \cite{icct2022}, have introduced an alternative approach to obtain hard DTs, using an approximation during backpropagation to compute gradients using Straight Through Estimators (STEs) \cite{bengio2013estimating,hubara2016binarized}. This approximation can hamper DT training, especially in large datasets or RL, where errors may accumulate over many training steps.

In this work, we propose a powerful novel encoding of oblique DTs as NNs, namely the {\em DTSemNet} architecture, which overcomes the aforementioned shortcomings. It uses \texttt{ReLU} activation functions and linear operations, making it differentiable and allowing gradient descent to be applied to learn the structure. The encoding is semantically equivalent to a (hard) oblique DT architecture, such that decisions (weights) in the DT correspond one-to-one to the trainable weights in the NN.
The other weights are fixed (non-trainable) in the architecture. The primary use of {\em DTSemNet} is as a classifier; that is, it provides a class associated with a given input, using a typical \texttt{argmax} operation to select the class associated with the highest output. We show in Theorem \ref{th1} that {\em DTSemNet} is exactly equivalent to a DT; that is, for every input vector, the output class produced by {\em DTSemNet} and by the DT are the same. It has the potential for application in many different contexts to learn DTs. Compared with DGT \cite{dgt2022}, and ICCT \cite{icct2022}, both needing {\em many} STE calls, {\em DTSemNet} uses \texttt{ReLU} and {\em standard} gradient descent
without STE approximation. We propose a simple extension to {\em DTSemNet} for {\em regression tasks}, which requires {\em one} STE call to combine the choice of value with the choice of the leaf. Each leaf is associated with a regressor whose parameters are learned simultaneously with the DT decisions via (standard) gradient descent. This process results in DTs that generalize, aided by the regressors at the leaf, similar to \cite{icct2022},\cite{tao-regression}.

We experimented with {\em DTSemNet} in various benchmarks, both for classification and regression tasks, and compared its performance against different methods producing hard DTs: DGT\cite{dgt2022} for gradient descent-based, TAO \citep{carreira2018alternating, tao-regression} for non-greedy algorithm, CRO-DT \cite{ea2023} for global-searches and CART as the standard for greedy algorithms.
The classification and regression tasks mostly involve datasets with tabular data, except for MNIST, a small-sized image dataset. The number of features varies from 4 to 780 (outlier for MNIST), and the number of classes varies from 2 to 26. For classification tasks, {\em DTSemNet} achieves the best performance on every dataset, showing the efficiency of our proposed {\em unapproximated} methodology. For regression tasks, {\em DTSemNet} is either the best performing or second (e.g., beaten by TAO). One plausible cause of this discrepancy is the STE approximation used for {\em regression tasks} to combine the decision in the tree with the value computed at the nodes. This approximation sometimes hurts efficiency, whereas no such approximation for classification results in optimal performance. In both cases, DT learning time is much shorter using gradient descent than competing solutions.

Last, we explain how to deploy {\em DTSemNet} to learn DT policies in RL environments, both with discrete action spaces ({\em DTSemNet}-classification) and with continuous action spaces ({\em DTSemNet}-regression). 
We accomplish this by simply replacing the NN architecture with the {\em DTSemNet} architecture in RL learning pipelines (e.g., PPO \cite{schulman2017proximal}), to obtain DT policies having performance comparable to that of NNs. Experimentation on RL environments with {\em discrete} action spaces of different sizes (up to 10 actions and 32 input dimensions) shows that {\em DTSemNet} matches NNs. Further, on the more complicated benchmarks, it significantly outperforms competing solutions producing DT policies, either obtained by gradient descent on architectures encoding DTs or by imitation learning of NNs.
For {\em continuous} action spaces, provided that the number of input dimensions is limited,  ICCT already provides performance matching NN by using regressors \cite{icct2022}, although it generates less expressive axis-aligned DTs. This suggests that very accurate decisions in internal nodes are not as important when regressors can be used in the leaves of the DT. We experimentally verified that in two environments with continuous action spaces: as with discrete actions, {\em DTSemNet} outperforms competing solutions, but the architecture is less important for performance.

\smallskip
\noindent 
\textbf{Main contributions} are as follows:
\begin{itemize}
 \item We introduce the {\em DTSemNet} architecture, which we prove to be semantically equivalent to DTs. 
 \item Learning a {\em DTSemNet} (-classification) with standard gradient descent corresponds exactly to learning a DT, without resorting to any approximations, unlike competing methods. Experimentally, this results in the most efficient DTs for classification tasks.
 \item We extend {\em DTSemNet} to regression tasks, using one STE approximation to combine the binary decisions with an output regression: experimentally, {\em DTSemNet}-regression is sometimes the best method, sometimes (a close) second.
 \item  We explain how to use {\em DTSemNet} in the RL setting, again producing experimentally the best DT policies, by a large margin for discrete actions ({\em DTSemNet}-classification), and small margins for continuous action spaces ({\em DTSemNet}-regression).
 \end{itemize}

\section{Related Work}\label{related_work}
\textbf{Non-Gradient-Based DT Training:} There are many approaches to learning DTs without relying on gradients. CART \citep{breiman1984cart} (and its extensions) is a well-known method for training ({\em axis-aligned}) DTs, which is based on splitting the dataset at each node based on certain features using a metric, such as entropy or Gini impurity. 
Concerning learning {\em oblique} DTs in this way, some methods have been proposed, such as Oblique Classifier 1 (OC1) \cite{murthy1993oc1} or GUIDE \cite{loh2014fifty}. They typically do not produce well-performing oblique DTs since learning oblique DT is much harder than learning axis-aligned DTs.
TAO is currently a state-of-the-art (SOTA) method that can be used to learn oblique DTs. It enhances the performance of DT obtained by CART (or random DT of given depth) by alternatively fine-tuning node parameters at specific depths \cite{carreira2018alternating, tao-regression}. 

Concerning MIP formulations \cite{bennett1996optimal, bertsimas2017optimal} 
or Global EA-based search approaches, such as CRO-DT \cite{ea2023}, they learn DTs by searching over various structures of DTs but at high computational costs, which is impractical for large DTs and datasets. 
The proposed {\em DTSemNet} overcomes these challenges by using gradient-descent to lower training time, which we confirm by comparing with CRO-DT \cite{ea2023}, which proposes matrix encoding of (oblique) DTs to speed up training compared to previous EA-based methods, and produces axis-aligned DTs.

\smallskip
\noindent
\textbf{Gradient-Based DT Training:}
Numerous works propose to approximate DTs as soft-DTs to use gradient descent for learning DTs, where decision nodes typically use the \texttt{Sigmoid} function \cite{zantedeschi2021learning,hazimeh2020tree,yang2018deep,frosst2017distilling,tanno2019adaptive,biau2019neural,silva2020_softDT,ding2021cdt,qiu2021programmatic,policy_tree,neural_forestICCV,samuel,wan2020nbdt}. {\em Hardening} soft DTs, 
that is, transforming them into hard DTs by discretizing the probabilities induces severe inaccuracies \cite{icct2022}.
More closely related to our work, DGT \cite{dgt2022} represents (oblique) DTs as an NN-architecture using the (non-differentiable) sign activation function, resorting to {\em quantized} gradient descent to learn it, leveraging principles from training binarized NNs using STE \cite{hubara2016binarized}. 
Specifically, during forward propagation, nodes utilize a 0-1 step function, whereas, during backward propagation, nodes employ a piecewise linear function or some differentiable approximation (see [22]). Similarly, ICCT \cite{icct2022} learns (axis-aligned) DTs using NNs with the \texttt{Sigmoid} activation function and STEs. In all these works, the hard DTs that are produced are (slightly) different from the DT (soft DT or using STE), which is being optimized by gradient descent. In contrast, the {\em DTSemNet} architecture using ReLU activation functions allows standard gradient descent to be performed, and the output DT from {\em DTSemNet}-classification is exactly the same as the function optimized by gradient-descent, without approximation. Experiments confirm that it is more accurate in practice, significantly so for classification tasks.

\smallskip
\noindent
\textbf{Training DT policies in RL Setup:}
Hard or soft DT policies can be obtained via imitation learning \cite{imitationlearning}, i.e., learning from expert policies, usually pretrained NNs \cite{bastani2018verifiable,jhunjhunwala,liu2019toward,bewley2021tripletree,roth2019conservative,verma2018programmatically,verma2019imitation,coppens2019distilling}. For instance, VIPER \cite{bastani2018verifiable} imitates a Q-network (or policy network) by creating a dataset from collected samples and trains a DT using CART, with sample weightage assigned based on Q-values. In contrast, {\em DTSemNet} directly learns a hard oblique DT in RL (using PPO \cite{schulman2017proximal}). 
Other works, such as ProLoNet \cite{prolonets}, that learn {\em soft} DTs using the RL framework, with the objective of initializing weights from expert humans. In contrast, {\em DTSemNet} learns a {\em hard} DT. 
ICCT \cite{icct2022} proposed an STE-based approach to learn axis-aligned DTs using gradient descent. By comparison, we can handle oblique trees, which are more expressive and accurate than axis-aligned DTs, especially for discrete actions.

\section{The {\em DTSemNet} Architecture}

\label{dtsement}

We now describe our main contribution, namely {\em \color{blue} DTSemNet}, which encodes decision trees ({\em \color{blue} DT}) as deep neural {\em \color{blue} Net}works (DNNs) in a {\em \color{blue} Sem}antically equivalent way. We prove the semantic equivalence in Theorem \ref{th1} thereafter.

Consider a (binary) DT $T$ over input space $X=\mathbb{R}^n$ of dimension $n$. 
It has a set of internal nodes $T_0, \cdots T_k$,
and leaves $T_{k+1}, \ldots, T_m$, with associated functions 
$\leftt, \rightt: \{T_0, \cdots T_k\} \rightarrow \{T_1, \cdots T_m\}$ associating 
each internal node with its respective left (or right) child. 
Each internal node $T_i$ is also associated with a decision described by a (linear) 
vector $A_i \in \mathbb{R}^{n} \setminus \{(0,\cdots, 0)\}$
and an (affine) component $b_i \in \mathbb{R}$, such that the decision from input $x \in X$ is to go to the right decision node iff $A_i x + b_i > 0$, and left if 
$A_i x + b_i < 0$ \footnote{Note that the case where
$A_i x + b_i = 0$ is not important: The set of initial points where $A_i x + b_i = 0$ for some $i$ have zero measure and thus zero probability happening. 
We consider that input space $X$ excludes such points.}.
There is always a unique path between a node and the root.
Using this path, we associate each leaf $T_j$ of $T$ with a set of decisions.
For instance, in Figure \ref{fig1}, leaf $T_5$ is associated with decisions $-D_0, -D_1, D_3$,
encoding the fact that $T_5$ is the right child of $T_3$ (it needs $D_3$ to be true, i.e.
$A_3 x + b_3 >0$), itself being the left child of node $T_1$ (and thus it needs 
$A_1 x + b_1 < 0$, i.e. $-A_1 x - b_1 > 0$, that we denote "$-D_1$"), with $T_1$ the left child of root $T_0$. Given an input $x\in X$, leaf $T_j$ is selected if and only if all the associated decisions are true.
Given an input $x \in X$, there is exactly one leaf for which all associated decisions are true, and for the other leaves, at least one associated decision is false.

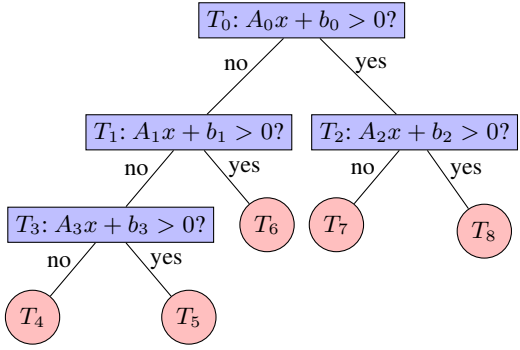
\begin{figure}[t]
  \centering
      \begin{tikzpicture}
        [ node/.style={ draw,  rectangle }, ]
    
        \node [node,fill=blue!25] (A) {$T_0$: $A_0 x + b_0 > 0$?};
        \path (A) ++(-135:21mm) node [node,fill=blue!25] (B) {$T_1$: $A_1 x + b_1 > 0$?};
        \path (A) ++(-45:21mm) node [node,fill=blue!25] (C) {$T_2$: $A_2 x + b_2 > 0$?};
        \path (B) ++(-130:16mm) node [node,fill=blue!25] (D) {$T_3$: $A_3 x + b_3 > 0$?};
    
        \path (B) ++(-50:16mm) node [node,circle,fill=red!25] (E) {$T_6$};
        \path (D) ++(-130:16mm) node [node,circle,fill=red!25] (F) {$T_4$};
        \path (D) ++(-50:16mm) node [node,circle,fill=red!25] (G) {$T_5$};
        \path (C) ++(-130:16mm) node [node,circle,fill=red!25] (H) {$T_7$};
        \path (C) ++(-55:16mm) node [node,circle,fill=red!25] (I) {$T_8$};
    
        \draw (A) -- (B) node [left,pos=0.35] {no}(A);
        \draw (A) -- (C) node [right,pos=0.35] {yes}(A);
        \draw (B) -- (D) node [left,pos=0.35] {no}(A);
        \draw (B) -- (E) node [right,pos=0.35] {yes}(A);
        \draw (D) -- (F) node [left,pos=0.35] {no}(A);
        \draw (D) -- (G) node [right,pos=0.35] {yes}(A);
        \draw (C) -- (H) node [left,pos=0.35] {no}(A);
        \draw (C) -- (I) node [right,pos=0.35] {yes}(A);

    \end{tikzpicture}
    \caption{A Decision Tree $T$ with 4 internal decision nodes $T_0, \ldots, T_3$
    and 5 leaves $T_4, \ldots, T_8$.}
    \label{fig1}
\end{figure}

Finding suitable values for $A_i$ and $b_i$ to make the DT $T$ efficient is challenging.
We explain in the following how to find such values using standard gradient descent algorithms.
For that, we encode the DT $T$ as the following NN architecture {\em DTSemNet},
which is semantically equivalent with $T$, and which can be converted back to an equivalent 
DT $T'$ after learning. We now describe the {\em DTSemNet} architecture:

\begin{itemize}

  \item The input layer is made of $n$ inputs $x_1, \ldots, x_n$ for the input domain $X$ of dimension $n$, plus one input $b$ for the biases, always set to weight $1$.
  
  \item The first hidden layer is a layer with $k$ nodes $I_1, \ldots I_k$, one per internal node $T_1, \ldots, T_k$ of the tree. It is fully connected with the initial layer. The weights to node $I_i$ correspond one-to-one to the decisions $D_i$, that is, to $A_i$ for the weights from $x_1, \ldots, x_n$ and to $b_i$ from $b$. The activation function is linear.
  
  \item The second hidden layer is made of $2k$ nodes $\top_i,\bot_i$, two per internal node 
  $T_i, i \leq k$ of the tree, standing for decision $D_i$ being true ($\top_i$), or $-D_i$ being true ($\bot_i$). The activation function is ReLU. The weights are fixed, with weight 1 from $I_i$ to $\top_i$ and $-1$ from $I_i$ to $\bot_i$.
  
  \item The output is made of $m-k$ nodes $L_{k+1}, \ldots, L_m$, one per leaf 
  $T_{k+1}, \ldots, T_m$ of the decision tree.

 The weights from $\bot_i,\top_i$ to $L_j$ can be 0 or 1, depending on the structure of the DT: $L_j$ takes weight $0$ from $\top_i$ if leaf $T_j$ is a left descendant of node $T_i$ (it needs $D_i$ to be false), otherwise it takes weight 1 from $\top_i$ (that is if either it is a right descendant of node $T_i$ or it is not a descendant of $T_i$ at all). Similarly, it takes 
 weight $0$ from $\bot_i$ if leaf $T_j$ is a right descendant of node $T_i$. Otherwise, it takes weight 1 from $\bot_i$.
The activation function is linear.
  \end{itemize}

The {\em DTSemNet} NN has $n+1+m+2k$ nodes, it has depth $4$ (2 linear activation functions and 1 ReLU activation function), with $k (n+1)$ trainable weights, exactly the same number as the number of dimensions of decisions in the DT (the $A_i$'s and $b_i$'s). We display in Fig.~\ref{fig2} the {\em DTSemNet} associated with the DT of Fig.~\ref{fig1}. 

The rationale behind the architecture is the following: for $\weight_x(X)$ the value of node $X$ when the NN input is $x$, we want to ensure that $\weight_x(L_\ell) = \sum_j \weight_x(\bot_j) + \weight_x(\top_j)$ for the leaf $T_\ell$ such that all the associated decisions are true (thus selected by the DT), and 
$\weight_x(L_{\ell'}) < \sum_j \weight_x(\bot_j)) + \weight_x(\top_j)$ for every other leaf $T_{\ell'} \neq T_\ell$, so that $L_\ell$ is selected by the 
{\em DTSemNet}. This allows to prove the semantic equivalence of the DT and the {\em DTSemNet} architecture:

\begin{figure}[t]
\center
\begin{tikzpicture}
  [shorten >=1pt,->,draw=black!50, node distance=\layersep]
    \tikzstyle{every pin edge}=[<-,shorten <=1pt]
    \tikzstyle{neuron}=[circle,fill=black!25,minimum size=17pt,inner sep=0pt]
    \tikzstyle{input neuron}=[neuron, fill=green!50]
    \tikzstyle{output neuron}=[neuron, fill=red!25]
    \tikzstyle{hidden neuron}=[circle,fill=blue!25,minimum size=17pt,inner sep=0pt]
    \tikzstyle{annot} = [text width=3em, text centered]
    
    \def\layersep{2.2cm}

    % Draw the input layer nodes
    \foreach \name / \y in {1,...,3}
        \node[input neuron, pin=left: $x_\y$] (I-\name) at (0, 0.5cm - 1.5*\y cm) {$x_\y$};

     \node[input neuron, pin=left: 1] (I-4) at (0,- 5.5 cm) {$b$};

    % Draw the hidden layer nodes
    \foreach \name / \y in {0,...,3}
        \path[yshift=0.5cm]
            node[hidden neuron] (HH-\name) at (\layersep,-1.5cm -1.5 * \y cm) {$I_\y$};

   % Draw the hidden layer nodes
    \foreach \name / \y in {0,...,3}
        \path[yshift=0.5cm]
            node[hidden neuron] (H-\name) at (2*\layersep,-1.15cm -1.5*\y cm) {$\top_\y$};

      \foreach \name / \y in {0,...,3}
          \path[yshift=0.5cm]
                node[hidden neuron] (B-\name) at (2*\layersep,-1.9 cm -1.5*\y cm) {$\bot_\y$};

    % Draw the output layer node

    \foreach \name / \y in {4,...,8}
    \path[yshift=0.5cm]
    node[output neuron] (O-\name) at (3*\layersep,3.5 cm -1.2*\y cm) {$L_\y$};

    % Connect every node in the input layer with every node in the
    % hidden layer.
    \foreach \source in {1,...,3}
            \draw(I-\source) edge (HH-0);

            \draw(I-4) edge (HH-0);

     \foreach \source in {1,...,4}
            \foreach \dest in {1,...,3}
                \draw(I-\source) edge (HH-\dest);

    \foreach \source in {0,...,3}
       \draw (HH-\source) edge node[draw=none,fill=none,font=\scriptsize,midway,above] {1} (H-\source) ;

    \foreach \source in {0,...,3}
       \draw (HH-\source) edge node[draw=none,fill=none,font=\scriptsize,midway,below] {-1} (B-\source);
       
       \draw(B-0) edge (O-4);
       \draw(B-1) edge  (O-4);
       \draw(B-2) edge (O-4);
       \draw(H-2) edge (O-4);
       \draw(B-3) edge  (O-4);

       \draw(B-0) edge (O-5);
       \draw(B-1) edge (O-5);
       \draw(B-2) edge  (O-5);
       \draw(H-2) edge  (O-5);
       \draw(H-3) edge  (O-5);

       \draw(B-0) edge  (O-6);
       \draw(H-1) edge (O-6);
       \draw(B-2) edge (O-6);
       \draw(H-2) edge (O-6);
       \draw(B-3) edge  (O-6);
       \draw(H-3) edge  (O-6);

       \draw(H-0) edge (O-7);
       
       \draw(H-1) edge  (O-7);
       \draw(B-2) edge  (O-7);
       \draw(B-3) edge  (O-7);
       \draw(H-3) edge  (O-7);

       \draw(H-0) edge  (O-8);
       \draw(B-1) edge  (O-8);
       \draw(H-1) edge  (O-8);
       \draw(H-2) edge  (O-8);
       \draw(B-3) edge (O-8);
       \draw(H-3) edge  (O-8);

       \draw(B-1) edge node[draw=none,fill=white,midway,above] {1} (O-7);
       
       \draw(I-3) edge node[draw=none,fill=white,midway,above] {$A_i,b_i$} (HH-1);

    % Annotate the layers
    
    \node[annot,above of=H-0, node distance=0.5cm] (hl) {ReLU};

\end{tikzpicture}
\caption{The {\em DTSemNet} architecture corresponding to Decision Tree $T$ in Figure \ref{fig1}, with $X=\mathbb{R}^3$}
\label{fig2}
\end{figure}
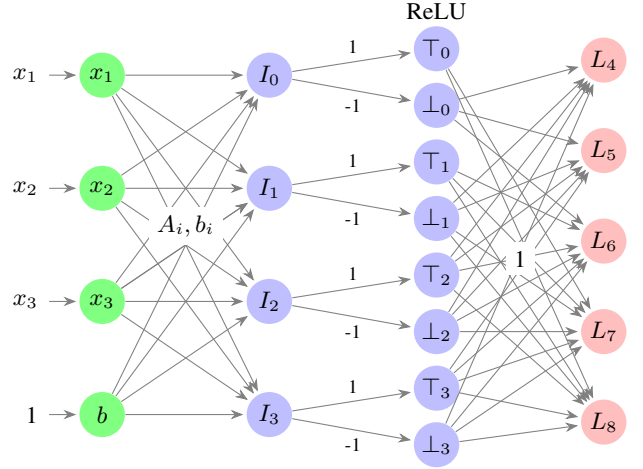

\begin{theorem}
\label{th1}
Considering the Neural Network {\em DTSemNet} as a Classifier, taking the decision associated with the output $L_i$ with the highest value using the standard argmax operation, gives exactly the same decision as the DT, i.e., select the (unique) leaf for which every associated decision is true.
Equivalently, For all input $x \in X$, we have $argmax_j(\weight_x(L_j))=L_i$ iff every decision associated with leaf $T_i$ are true on input $x$.
\end{theorem}

\begin{proof}
  Take some $x \in X$.
  Let us start with a few important remarks. First, exactly one leaf $T_i$ has all its associated decisions true (and is the leaf selected by the DT). For every other leaf $T_j, j \neq i$ there is at least one decision $D_j^x$ (some $D_\ell$ or some $-D_\ell$) associated with $T_j$, which is false.

Also, trivially, we have that exactly one of $\weight_x(I_i), -\weight_x(I_i)$ is strictly positive, and the other one is strictly negative.
Thus, because of the definition of ReLU, we have that exactly one of $\weight_x(\bot_i), \weight_x(\top_i)$ is 0, and the other is strictly positive. In particular, 
$0 \leq \weight_x(\bot_i))$, $\weight_x(\top_i) \leq  
\weight_x(\bot_i)) + \weight_x(\top_i) = \max(\weight_x(\bot_i), \weight_x(\top_i))$.

For all leaf $T_\ell$, we have 
$\weight_x(L_\ell) \leq \sum_j \weight_x(\bot_j)) + \weight_x(\top_j)$ by construction:
for all node $T_j$ not in the path between the root $T_0$ and $T_\ell$, 
these two weights are input into $L_k$ by construction (with weight 1), and for nodes $T_j$ that are on the path, one of them is input to $L_k$, and the other one is positive or null.

Let $\ell$ be such that every decision associated with leaf $T_\ell$ is true on input $x$. 
There is exactly one such leaf $T_\ell$.
We show that $\weight_x(L_\ell) = \sum_j \weight_x(\bot_j) + \weight_x(\top_j)$.
Indeed, as above, the only nodes $\bot_j$ (resp. $\top_j$) that are not direct input of 
$L_\ell$ corresponds to decisions that are not true, and thus 
$\weight_x(\bot_j)=0$ (resp. $\weight_x(\top_j)=0$).
Hence, the equality is true.

For every other leaf $L_{\ell'} \neq L_\ell$, we show that 
$\weight_x(L_{\ell'}) < \weight_x(L_\ell)$.
Indeed, let $D^x_{\ell'}$ be a decision associated with $T_{\ell'}$ that is not correct under input $x$, 
say it is decision $-D_i$, with $D_i$ true, that is $\weight_x(\top_i)>0$. As there is no weight from $\top_i$ on $L_{\ell'}$ by definition of {\em DTSemNet},
$\weight_x(L_{\ell'}) \leq (\sum_j \weight_x(\bot_j) + \weight_x(\top_j)) - 
\weight_x(\top_j) < \weight_x(L_\ell)$.
That is, $argmax_j(\weight_x(L_j))=L_i$ iff $i= \ell$, and we are done.
\end{proof}

While CART \cite{breiman1984cart} grows the DT incrementally, 
{\em DTSemNet} will learn a fixed DT architecture. The simplest way to fix the architecture is to pick a depth $n$ in regards to the difficulty of the task (number of features, etc., e.g., $n=5$ for few $\approx 5$ features) and consider the balanced binary tree with $2^n$ leaves ($2^5=32$). This is as simple as choosing the number of neurons of an NN. Notice that unlike DGT \cite{dgt2022}, {\em DTSemNet} can also directly handle any unbalanced trees, e.g., for optimizing the number of nodes.

%%%%%%%%%%%%%%%%%%%%%%%%%%%%%%%%%%%%%%%%%%%
%%%%%%%%%%% Classification %%%%%%%%%%%%%%%
%%%%%%%%%%%%%%%%%%%%%%%%%%%%%%%%%%%%%%%%%%%
\noindent
\paragraph{{\em DTSemNet}-classification:}

We now explain how the {\em DTSemNet} architecture selects the class for a classification task. Let $T_{class}$ a balanced DT with one leaf for each class. For instance, for {\em balance scale}, there are 3 classes {\em left}, {\em right} and {\em balanced}. The balanced DT is provided below (see Figure \ref{fig3}). The last (here 2) levels of the DT to learn are then replaced by $T_{class}$, letting the gradient descent learn the most efficient way to play actions from each region of the space. We transform this tree using the {\em DTSemNet} architecture: there will thus be several nodes $L_j$ corresponding to leaves $T_j$ associated with the same class.
To recover the standard classification framework in {\em DTSemNet}, we add a final layer with $\ell$ nodes $C_1, \ldots, C_\ell$, one for each of the $\ell$ classes $\alpha_1, \ldots, \alpha_\ell$. The input weight to $C_i$ is 1 from $L_j$ if leaf $T_j$ is associated with action $\alpha_i$, and 0 otherwise. The \texttt{MaxPool} operation is performed at the node $C_i$, which then outputs the maximum value of associated inputs into node $C_i$.
Applying Theorem \ref{th1}, if leaf $T_j$ is selected by the DT on input $x \in X$, then the value $\weight_x(L_j)$ is maximal, and for $\alpha_i$ the class associated with $T_j$, output node $C_i$ will get the highest value among all output nodes $C_j$. Class $\alpha_i$ will be thus selected by {\em DTSemNet}, semantically equivalent to the DT $T$.

\begin{figure}[]
  \centering

  \begin{tikzpicture}
    [ node/.style={ draw,  rectangle }, ]

    \node [node,fill=blue!25] (A) {$T_0$};
    \path (A) ++(-150:16mm) node [node,fill=blue!25] (B) {$T_1$};
    \path (A) ++(-30:20mm) node [node,fill=red!25] (C) {balanced};
    \path (B) ++(-120:16mm) node [node,fill=red!25] (FF) {left};
    \path (B) ++(-60:16mm) node [node,fill=red!25] (G) {right};
    
    \draw (A) -- (B) node [left,pos=0.35] {} (A);
    \draw (A) -- (C) node [right,pos=0.5] {} (A);
    \draw (B) -- (FF) node [left,pos=0.35] {}(A);
    \draw (B) -- (G) node [right,pos=0.35] {}(A);

\end{tikzpicture}
\caption{Decision Tree $T_{class}$ for Balance scale.}
\label{fig3}
\end{figure}
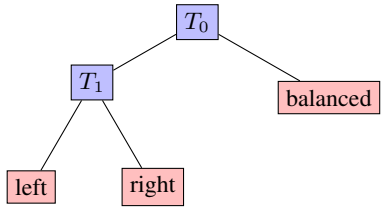

%%%%%%%%%%%%%%%%%%%%%%%%%%%%%%%%%%%%%%%%%%%
%%%%%%%%%%% Regression %%%%%%%%%%%%%%%
%%%%%%%%%%%%%%%%%%%%%%%%%%%%%%%%%%%%%%%%%%%

\noindent
\paragraph{{\em DTSemNet}-regression:}

Usual DT for regression tasks learned by CART \cite{breiman1984cart} (axis-aligned DTs) and by DGT \cite{dgt2022} (oblique DTs) are as follows: each leaf $T_j$ is associated with one scalar value $\alpha_j$, and the DT $T$ outputs $\alpha_\ell$ for $T_\ell$ the leaf selected by the decisions of $T$ on input $x$. In CART, this parameter $\alpha_\ell$ is the average of the training values corresponding to inputs ending in $T_\ell$, resulting in a DT $T$ that does not generalize, even when learning a simple linear function.

We now describe regression-DTs, which are more expressive, as used in 
TAO(-linear) \cite{tao-regression} and ICCT \cite{icct2022}.
Every leaf $T_j$ of the (regression-)DT $T$ is associated with linear   parameter vector $\theta_j \in \mathbb{R}^{n}$ and an affine scalar $\alpha_j \in \mathbb{R}$. The output of regression-DT $T$ on input $x\in X$ is $\theta_\ell \cdot x + \alpha_\ell$, where $T_\ell$ is the leaf selected on input $x$ by the decisions of $T$. A regression-DT would easily generalize two or more data points from a training set into the linear function. Regression-DTs actually correspond to a piecewise linear regression, where the different linear pieces are described by the parameters at the leaves, and the switch between pieces by the decisions of the DT.

\begin{figure}[]
  \centering
  \includegraphics[width=0.9\linewidth]{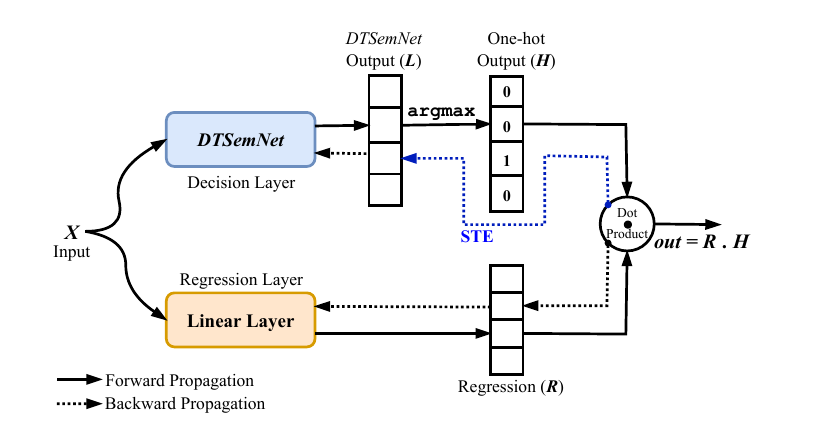}
  \caption{{\em DTSemNet}-regression adapted for regression tasks.}
  \label{fig:regression}
\end{figure}

We now adapt {\em DTSemNet} to learn such a regression-DT.
We learn simultaneously: (a) the parameters of the linear regression at every leaf and (b) the decisions leading to the most appropriate leaf and associated linear regression. The first hidden (linear) layer is extended with $m-k$ nodes $(R_j)_{k < j \leq m}$, one per leaf $T_j$, with trainable weights $\mathbf{\theta}_j$ from input $x$ to $R_j$ and $\alpha_j$ from input $1$ to $R_j$: We have $\weight_x(R_j) = \mathbf{\theta}_j \cdot x + \alpha_j$. In order to select $\weight_x(R_\ell)$ for $L_\ell$ with the highest 
value in {\em DTSemNet} (that is $T_\ell$ selected by the DT by Theorem \ref{th1}),
we add a one-hot layer with $m-k$ nodes $(H_j)_{k < j \leq m}$ using the \texttt{argmax} function (see Figure \ref{fig:regression}): node $H_\ell$ has value 1 for $L_\ell$ with maximal value, and $H_j$ has value 0 for $j \neq \ell$; finally an output node $out$ computes the scalar product between $(value_x(H_j))_{k < j \leq m}$ and $(value_x(R_j))_{k < j \leq m}$, obtaining $\weight_x(out) = \weight_x(R_\ell) =\mathbf{\theta}_j \cdot x + \alpha_j$.
In end-to-end training, the \texttt{argmax} operation makes the entire architecture non-differentiable, so STE 
\cite{bengio2013estimating,hubara2016binarized} is used to bypass the gradient through \texttt{argmax} operation during backpropagation. The STE approximation is only used for {\em DTSemNet}-regression (not -classification), and only at the output node, in contrast to STE approximation, which is used at every single node in DGT and ICCT \cite{dgt2022, icct2022}, and even for classification.

%%%%%%%%%%%%%%%%%%%%%%%%%%%%%%%%%%%%%%%%%%%
%%%%%%%%%%% RL %%%%%%%%%%%%%%%
%%%%%%%%%%%%%%%%%%%%%%%%%%%%%%%%%%%%%%%%%%%
\paragraph{{\em DTSemNet} for RL:}

To tackle RL tasks, we use the standard Deep RL framework (e.g., PPO \cite{schulman2017proximal}), simply replacing the NN with the {\em DTSemNet} architecture and running the gradient descent from the RL framework. For environments with discrete action spaces, we use {\em DTSemNet}-classification, with one class per action. For environments with continuous action spaces, we use {\em DTSemNet}-regression, with one linear regression per action dimension (e.g., 2 for a continuous lunar lander, one for horizontal and one for vertical engines)

%%%%%%%%%%%%%%%%%%%
%%%% TABLES %%%%%%%
%%%%%%%%%%%%%%%%%%%
\begin{table*}[t!]
    \centering
    \caption{Percentage accuracy (higher is better) on {\bf classification tasks} with the most accurate height $\leq 4$, as tested in \cite{ea2023}. For each dataset, we provide the number of features $N_f$, classes $N_c$ and training samples $N_s$. Averaged accuracy $\pm$ std is reported over 100 runs. The best-performing height is reported in parentheses for each architecture. We run the \textit{DTSemNet} and DGT experiments, while the results for TAO, CART and CRO-DT are copied from \cite{ea2023}.}
    \begin{tabular}{llccccc}
    \toprule
    Dataset & $N_f$, $N_c$, $N_s$ & {\em DTSemNet} & DGT \cite{dgt2022} & TAO \cite{carreira2018alternating} & CART \cite{breiman1984cart} & CRO-DT \cite{ea2023} \\
    \midrule 
    Balance Scale & 4, 3,625 & $\mathbf{90.2 \pm 2.2}$ (2) & $88.6 \pm 1.7$ (4) & $77.4 \pm 3.1$ (4) & $74.9 \pm 3.6$ (4) & $77.8 \pm 3.0$ (4) \\
    Banknote Auth & 4, 4, 1372 & $\mathbf{99.8 \pm 0.4}$ (3) & $ \mathbf{99.8 \pm 0.4}$ (3) & $96.6 \pm 1.3$ (4) & $93.6 \pm 2.2$ (4) & $95.2 \pm 2.2$ (4)\\
    Blood Transfusion & 4, 2,784 & $\mathbf{78.5 \pm 1.7}$ (2) & $78.3 \pm 2.4$ (4)& $76.9 \pm 2.1$ (3) & $77.1 \pm 1.8$ (4) & $76.1 \pm 1.9$ (3) \\
    Acute Inflam. 1 & 6, 2, 120 & $\mathbf{100 \pm 0.0}$ (2) & $\mathbf{100 \pm 0.0}$ (4) & $99.7 \pm 1.2$ (4) & $\mathbf{100 \pm 0.0}$ (3) & $\mathbf{100 \pm 0.0}$ (2)\\
    Acute Inflam. 2 & 6, 2, 120 & $\mathbf{100 \pm 0.0}$ (2) & $\mathbf{100 \pm 0.0}$ (4) & $99.0 \pm 2.6$ (2) & $99.0 \pm 2.6$ (2) & $\mathbf{100 \pm 0.0}$  (2)\\
    Car Evaluation & 6, 4, 1728 & $\mathbf{93.3 \pm 2.2}$ (4) & $92.1 \pm 2.4$ (4) & $84.5 \pm 1.5$ (4) & $84.3 \pm 1.4$ (4) & $86.1 \pm 1.3$ (4) \\
    
    Breast Cancer & 9, 2, 683 & $\mathbf{97.2 \pm 1.3}$ (2) & $\mathbf{97.2 \pm 1.2}$ (2) & $94.7\pm 1.6$ (3) & $94.7 \pm 1.7$ (3) & $95.5 \pm 1.8$ (2) \\
    Avila Bible & 10, 12, 10430 & $\mathbf{62.2 \pm 1.4}$ (4) & $59.7 \pm 1.8$ (4) & $55.8 \pm 0.8 $ (4) & $54.0 \pm 1.3$ (4) & $59.6 \pm 0.7$ (4) \\
    Wine Quality Red & 11, 6, 1599 & $ \mathbf{58.6 \pm 2.2}$ (3) & $56.6 \pm 1.4$ (4)& $56.9 \pm 2.5$ (4) & $55.9 \pm 2.3$ (4) & $55.8 \pm 2.2$ (2) \\
    Wine Quality White & 11, 7, 4898 & $\mathbf{53.5 \pm 1.4}$ (4) & $52.1 \pm 1.6$ (4) & $52.3 \pm 1.4$ (4) & $52.0 \pm 1.3$ (4) & $51.4 \pm 1.2 $ (2) \\ 
    Dry Bean & 16, 7, 13611 & $\mathbf{91.4 \pm 0.5}$ (4) & $89.0 \pm 1.6$ (4) & $83.2 \pm 1.5$ (4) & $80.5 \pm 1.9$ (4) & $77.9 \pm 4.7$ (4) \\
    Climate Crashes & 18, 2, 540 & $\mathbf{92.9 \pm 1.4}$ (2) & $92.4 \pm 2.4$ (3) & $90.6 \pm 2.2$ (3) & $91.8 \pm 1.8$ (4) & $91.5 \pm 2.0$ (2) \\
    Conn. Sonar & 60, 2, 208 & $\mathbf{82.1 \pm 5.1}$ (4) & $80.8 \pm 5.3$ (4) & $70.9 \pm 5.8$ (4) & $70.6 \pm 6.6$ (4) & $71.7 \pm 6.7$ (4) \\
    Optical Recognition & 64, 10, 3823 & $\mathbf{93.3 \pm 1.0}$ (4) & $91.9 \pm 1.0$ (4) & $64.6 \pm 6.5$ (4) & $53.2 \pm 3.2$ (4) & $65.2 \pm 2.0$ (4) \\
    \bottomrule
    \end{tabular}
    \label{tab:classification2}
\end{table*}

\paragraph{Comparison with other encodings of DTs into NNs:}

DGT \cite{dgt2022} uses the \texttt{Sign} activation function at each layer to produce (hard) oblique DTs. As \texttt{Sign} has no gradient (unlike \texttt{ReLU} used in {\em DTSemNet}), DGT resorts to {\em quantized} gradient descent, using an STE approximation process at every node. In regression tasks, DGT produces a scalar for each leaf instead of a linear regression for {\em DTSemNet}-regression.
The ICCT architecture\cite{icct2022} generates 
axis-aligned DTs for RL tasks by associating every leaf with the product of the weights of the edges leading to that leaf, passing in logarithmic space to handle the explicit multiplications.
The \texttt{Sigmoid} activation is used on each decision, incurring an approximation of the decision as a soft DT. The soft DT is crispified at each step of the RL process into a hard (axis-aligned) DT, using an STE to backpropagate through a non-differentiable function (Heaviside step function), which {\em DTSemNet}-classification avoids.

\section{Experimental Evaluation}\label{sec:results}
In this section, we evaluate the performance of {\em DTSemNet}, comparing it with competing methodologies learning hard DTs. 
Firstly, we consider supervised learning setups using multiple benchmark multi-class classification and regression datasets, on which we compare the accuracy on test data with the SOTA non-greedy method TAO \cite{carreira2018alternating, tao-regression}, the SOTA gradient descent-based method DGT \cite{dgt2022}, both learning oblique DTs, as well as CRO-DT \cite{ea2023} for global-searches and CART as the standard for greedy algorithms, both learning axis-aligned DTs. The relative training times are reported for benchmarks for which it is available (notice that we cannot run TAO on our hardware as it is not openly available). Further, to understand the impact on the generalization of different architectures using gradient descent, we leverage insights from the loss landscape \cite{li2018visualizing}. 

Lastly, we consider RL environments, both with discrete actions and continuous action spaces. We compare {\em DTSemNet} with DGT, both of which generate oblique DT-policies, ICCT, which generates axis-aligned DT-policies, all three through gradient descent, as well as VIPER, which generates axis-aligned DT-policies through imitation learning of an NN policy generated by Deep RL, which we also report as a baseline. 

We implemented {\em DTSemNet} and conducted all experiments using Python and PyTorch. Our testing platform has 8 CPU cores (AMD 75F3, Zen 3 architecture), 128 GB of RAM, and a 2 GB GPU (NVIDIA Quadro P620). The supplementary material \cite{dtsemnet-supp} provides additional results and details regarding datasets, train-test splits, hyperparameters, etc. 
Our source code is publicly available at \textit{\url{https://github.com/CPS-research-group/dtsemnet}}.

\medskip
\noindent
\textbf{Classification Tasks (small DTs):} We first consider the 14 classification tabular datasets used in CRO-DT \cite{ea2023}.
Global searches such as CRO-DT \cite{ea2023} are efficient only for small DTs (here up to depth 4, i.e. 32 nodes). We sort the benchmarks by the number of features, as it is more challenging to decide over more features, all the more so with small DTs. We report in Table \ref{tab:classification2} the (average) score over 100 DTs learned with different seeds for the most accurate height (up to $4$). 

In every single benchmark, {\em DTSemNet} produces the most accurate DTs, with the biggest difference in {\em Dry Beans}, reducing the classification error from $11\%$ to $8.6\%$. DGT is the second best, except in the 2 wine quality benchmarks, where TAO outperforms it. DGT usually comes close to {\em DTSemNet}, which is not surprising as the idea is similar, though the approximations used (STEs, quantized gradient descent) make DGT perform on average $0.85\%$ worse than {\em DTSemNet}. Interestingly, the advantages grow to $1.3\%$ on the 4 benchmarks with the most features (the hardest ones). Further, DGT needs larger trees than {\em DTSemNet} to get its best results (in $6$ out of $14$ benchmarks). Compared with non gradient-based learning, the advantage of {\em DTSemNet} is very tangible, $5.5\%$ on average, growing to $12\%$ on the 4 benchmarks with the most features. The highest difference is in {\em Optical Recognition}, reducing the classification error from $34.8\%$ to just $6.7\%$.

\begin{table}[]
\caption{Training times in seconds (lower is better), and (avg accuracy $\%$ (higher is better)). The MNIST training time from the non-publicly available TAO is quoted from \cite{carreira2018alternating}, while we train other architecture on a similar compute configuration. The tree height is 8 for MNIST and 4 for DryBean.}
\centering
    \begin{tabular}{lllll}
    \toprule
    Dataset & {\em DTSemNet} & DGT \cite{dgt2022} & TAO \cite{arora2018optimization} & CRO-DT \cite{ea2023} \\
    \midrule
    MNIST & $306$ ($96.1$) & $288$ ($94.0$) & $1200$ ($95.0$) & $4659$ ($58.2$) \\
    DryBean & $4.4$ ($91.4$) & $3.8$ ($89.0$) & NA ($83.2$) & $1300$ ($77.9$) \\
    \bottomrule
    \end{tabular}
\label{tab:time}
\end{table}

\begin{table*}[t!]
\caption{Percentage accuracy (higher is better) on {\bf classification tasks} for the datasets reported in \cite{carreira2018alternating}. 
For each dataset, we provide the number of features $N_f$, classes $N_c$, and training samples $N_s$. All methods use the tree height fixed in \cite{dgt2022} (except CART, which has no predefined height). 
Averaged accuracy $\pm$ std is reported over 10 runs. The results for DGT and CART are from \cite{dgt2022}, and the results of TAO are from \cite{carreira2018alternating}.}
\centering
\begin{tabular}{llccccc}
\toprule
Dataset & $N_f$, $N_c$, $N_s$ & Height & {\em DTSemNet} & DGT \cite{dgt2022} & TAO \cite{carreira2018alternating} & CART \cite{breiman1984cart} \\
\midrule
Protein & 357, 3, 14895 & 4 & $\mathbf{68.60 \pm 0.22}$ & $67.80 \pm 0.40$ & $68.41 \pm 0.27$ & $57.53 \pm 0.00$ \\
SatImages & 36, 6, 3104 & 6 & $\mathbf{87.55 \pm 0.59}$ & $86.64 \pm 0.95$ & $87.41 \pm 0.33$ & $84.18 \pm 0.30$ \\
Segment & 19, 7, 1478 & 8 & $\mathbf{96.10 \pm 0.53}$ & $95.86 \pm 1.16$ & $95.01 \pm 0.86$ & $94.23 \pm 0.86$ \\
Pendigits & 16, 10, 5995 & 8 & $\mathbf{97.02 \pm 0.32}$ & $96.36 \pm 0.25$ & $96.08 \pm 0.34$ & $89.94 \pm 0.34$ \\
Connect4 & 126, 3, 43236 & 8 & $\mathbf{82.03 \pm 0.39}$ & $79.52 \pm 0.24$ & $81.21 \pm 0.25$ & $74.03 \pm 0.60$ \\
MNIST & 780, 10, 48000 & 8 & $\mathbf{96.16 \pm 0.14}$ & $94.00 \pm 0.36$ & $95.05 \pm 0.16 $ & $85.59 \pm 0.06$ \\
SensIT & 100, 3, 63058 & 10 & $\mathbf{84.29 \pm 0.11}$ & $83.67 \pm 0.23$ & $82.52 \pm 0.15$ & 
$78.31 \pm 0.00$ \\
Letter & 16, 26, 10500 & 10 & $\mathbf{89.19 \pm 0.29}$ & $86.13 \pm 0.72$ & $87.41 \pm 0.41$ & $70.13 \pm 0.08$ \\

\bottomrule
\end{tabular}
\label{tab:classification1}
\end{table*}

\begin{table*}[h!]
\caption{Average RMSE results (lower is better) on {\bf regression tasks}, 
$\pm$ std over 10 runs. The number $N_f$ of features and $N_s$ of training samples are provided for each dataset. DGT-linear is a modification we implemented from DGT \cite{dgt2022} to add regressors at the leaves.
The 5 first datasets are from \cite{tao-regression}.
The height for DTs with regressors at the leaves ({\em DTSemNet}, DGT-Linear, TAO-linear) is fixed to Height, following \cite{tao-regression}. DGT has fixed scalars at the leaves instead of regressors, thus it needs deeper DTs to achieve reasonable accuracy - the depth used in \cite{dgt2022}
is reported inside (), while CART height is not fixed. The last two datasets are from \cite{dgt2022}, for which  TAO-linear results are unavailable. The reported results for DGT and CART are taken from \cite{dgt2022}; TAO-Linear results from \cite{tao-regression}.}
    \centering
    \begin{tabular}{llcccccc}
    \toprule
    Dataset & $N_f$, $N_s$ & Height & {\em DTSemNet} & DGT-Linear & DGT \cite{dgt2022} & TAO-Linear \cite{tao-regression} & CART \cite{breiman1984cart} \\ 
    \midrule
    Abalone &10, 2004& 5 & $2.135 \pm 0.03$ & $2.144 \pm 0.03$ & $2.15 \pm 0.026$ (6) & $\mathbf{2.07 \pm 0.01}$ & $2.29 \pm 0.034$ \\
    Comp-Active &21, 3932& 5 & $2.645 \pm 0.18$ & $2.645 \pm 0.15$ & $2.91 \pm 0.149$ (6) & $\mathbf{2.58 \pm 0.02}$ & $3.35 \pm 0.221$ \\
    Ailerons &40, 5723& 5 & $\mathbf{1.66 \pm 0.01}$ & $1.67 \pm 0.017$ & $1.72 \pm 0.016$ (6) & $1.74 \pm 0.01$ & $2.01 \pm 0.00$ \\
    CTSlice &384, 34240& 5 & $1.45 \pm 0.12$ & $1.78 \pm 0.25$ & $2.30 \pm 0.166$ (10) & $\mathbf{1.16 \pm 0.02}$ & $5.78 \pm 0.224$ \\
    YearPred &90, 370972& 6 & $\mathbf{8.99 \pm 0.01}$ & $9.02 \pm 0.025$ & 
    $9.05 \pm 0.012$ (8) & $9.08 \pm 0.03$ & $9.69 \pm 0.00$ \\
\midrule
    PDBBind &2052, 9013 & 2 & $\mathbf{1.33 \pm 0.017}$ & $1.34 \pm 0.013$ & $1.39 \pm 0.017$ (6) & NA & $1.55 \pm 0.00$ \\
    Microsoft &136, 578729& 5 & $\mathbf{0.766 \pm 0.00}$ & $\mathbf{0.766 \pm 0.00}$ & $0.772 \pm 0.00$ (8) & NA & $0.771 \pm 0.00$ \\
       \bottomrule
    \end{tabular}
    \label{tab:regression}
\end{table*}

\medskip
\noindent
\textbf{Classification Tasks (deeper DTs):} We now consider the original 8 classification tasks reported in TAO \cite{carreira2018alternating} 
and reported in \cite{dgt2022}. They use tabular datasets, except for MNIST, which is a small-sized image dataset. We used the fixed tree height specified in 
\cite{carreira2018alternating} (also reused in \cite{dgt2022}).
We first report in Table \ref{tab:time} the training time for the architecture on MNIST, as \cite{carreira2018alternating} reports this number for TAO. We train other architectures on comparable computing configurations. We provide the accuracy number obtained for reference. We also report the training time for the simpler DryBean, for comparison sake, except for TAO, for which this is not available. 
First, CRO-DT for MNIST takes much longer to train (number of generations set to the default 4k) than other architectures while having very low accuracy numbers ($<60 \%$ instead of $>90\%$), the reason why we do not consider it for these benchmarks with deeper DTs (indeed, \cite{ea2023} does not report CRO-DT results for these benchmarks). The training times of DGT and {\em DTSemNet} are almost identical due to their architectural similarities. As expected, training with gradient-based methods is observed to be significantly faster than non-gradient learning methods.

We sort the benchmarks by the height of the DTs, which is a good indicator of the complexity of the tasks. We report in Table \ref{tab:classification1} the (average) score over 10 DTs learned with different seeds.
{\bf Discussion:} As in Table \ref{tab:classification2}, {\em DTSemNet} produces the most accurate DTs
in every single benchmark. With deeper trees, TAO recovers, and it is the second best 5 times vs 3 times for DGT. 
The advantage of {\em DTSemNet} over TAO is $1\%$ on average, growing to 
$1.4\%$ on the 4 benchmarks with deeper trees. 
The highest difference is in {\em Letter}, reducing the classification error from $12.6\%$ to $10.8\%$ and from $3.9\%$ to $3\%$ in {\em Pendigits}. The advantage of {\em DTSemNet} over DGT is $1.4\%$ on average, growing to $2.1\%$
on the 4 benchmarks with deeper trees, larger than for smaller trees from Table \ref{tab:classification2}. The largest difference is in {\em Letter}, reducing the classification error from $13.9\%$ to $10.8\%$, and 
from $6\%$ to $3.9\%$ in MNIST. The Friedman-Nemenyi test across all classification benchmarks, with a significance level of 0.05, shows that {\em DTSemNet} is significantly better than every other method with the average rank of 1.11 for {\em DTSemNet}, 2.25 for DGT, 2.86 for TAO, 3.60 for CRO-DT, and 3.77 for CART.

\begin{figure}[b]
  \centering
      \includegraphics[scale=0.29]{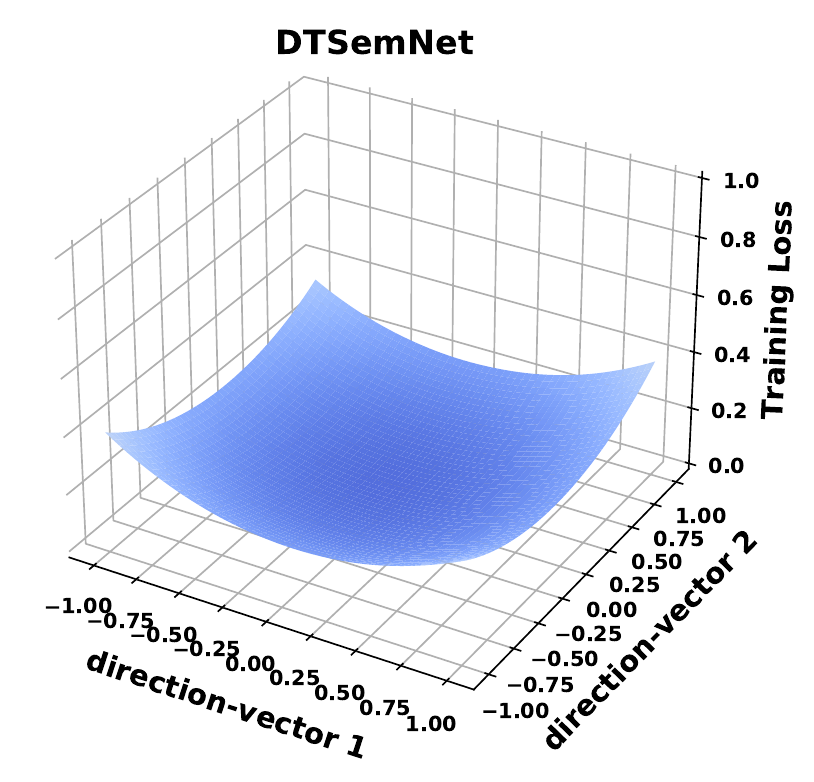}
      \includegraphics[scale=0.28]{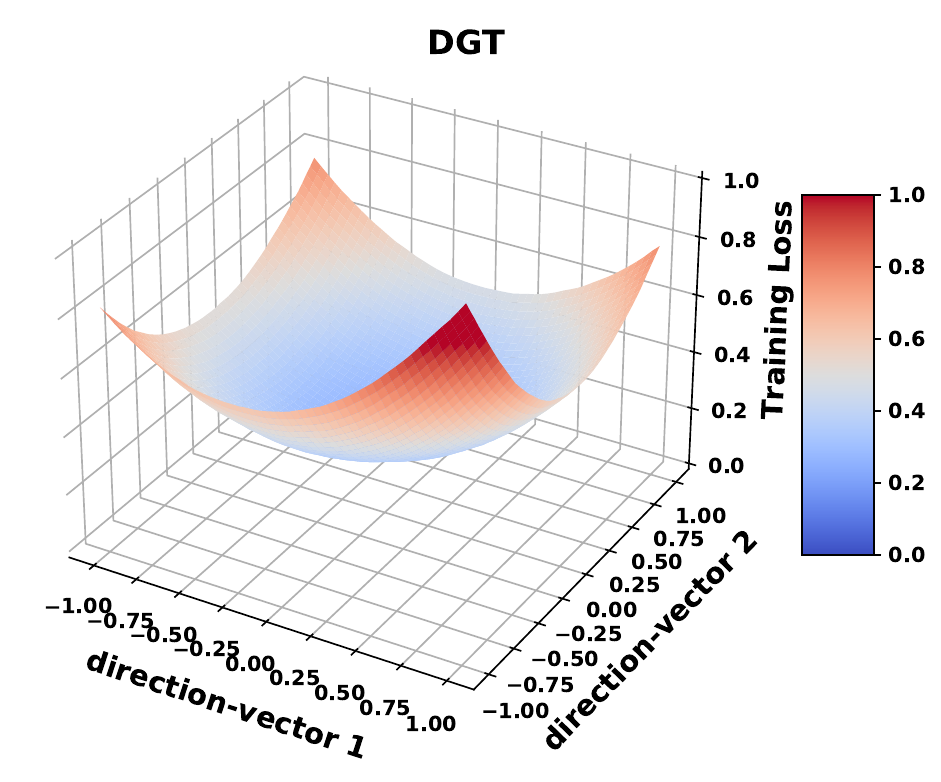}
      \caption{The loss landscape of {\em DTSemNet} and DGT for MNIST, when varying the parameters around the trained parameters along two random directions. The flatter the loss landscape, the better the generalization.}
      \label{fig:loss_ls}
\end{figure}

To understand the influence of approximations used in DGT (quantized gradient descent and STEs), we consider the loss landscape \cite{li2018visualizing}, which displays the effect of architecture choices on generalization. We consider MNIST, for which DGT makes many more classification errors than {\em DTSemNet}. Figure \ref{fig:loss_ls} displays the loss landscape around the final trained parameters along two random vector directions. A flatter loss landscape indicates better generalization capability \cite{li2018visualizing}. As shown in Figure \ref{fig:loss_ls}, the loss landscape of {\em DTSemNet} is very flat compared with DGT. This suggests that (approximation-free) {\em DTSemNet}  allows better generalization on classification tasks than DGT.

\begin{table*}[ht!]
\caption{Average rewards (higher is better) on {\bf RL tasks} $\pm$ std over 5 policies generated with different learning seeds. The policies are evaluated over 100 episodes. We report the number $N_f$ of features and 
$N_a$ of actions for each environment. The first four environments 
have discrete actions, and the bottom two continuous actions, for which we test both the original ({\em scalar}) version of DGT and DGT-linear ({\em linear}). For the 5 first environments, the {\em Height} is fixed for all architectures. For Bipedal Walker, DGT(-linear) performed much better with deeper DTs, while ICCT performed much better with shallower DTs ({\em DTSemNet} was not much impacted by the height). For Bipedal Walker, the heights yielding the best rewards for DGT(-linear) and ICCT are shown in parentheses.}
  \centering
  \begin{tabular}{llcccccc}
   \toprule
   Environments & $N_f$, $N_a$ & {\em Height} & {\em DTSemNet} & Deep RL & DGT \cite{dgt2022} & ICCT \cite{icct2022} & VIPER \cite{bastani2018verifiable}\\ 
   \midrule
    CartPole & 4, 2 & 4 & $\mathbf{500 \pm 0}$ & $\mathbf{500 \pm 0}$ & $\mathbf{500 \pm 0}$ & $496 \pm 0.3$ & $499.95 \pm 0.05$ \\
    Acrobot & 6, 3 & 4 & $\mathbf{-82.5 \pm 1.05}$ & $-84 \pm 0.84$ & $-83.1 \pm 1.88$  & $-88.6 \pm 1.77$ & $-83.92 \pm 1.59$ \\
    LunarLander & 8, 4 & 5 & $\mathbf{252.5 \pm 3.9}$ & $245 \pm 14.5$ & $183.6 \pm 14.6$ & $-85 \pm 16.3$ & $86.73 \pm 7.93$ \\
    Zerglings & 32, 30 & 6 & $\mathbf{15.54 \pm 2.07}$ & $10.47 \pm 0.23$ & $8.21 \pm 1.03$ & $9.40 \pm 1.10$ & $10.61 \pm 0.46$ \\
    \toprule
    Cont. LunarLander & 8, 2 dim. & 4 & $\mathbf{277.24 \pm 2.09}$ & $276.12 \pm 1.45$ & \begin{tabular}[c]{@{}l@{}} {\em scalar}: $131.92 \pm 51.49$ \\ {\em linear}: $267.9 \pm 9.37$ \end{tabular} & $255.57 \pm 4.19$ & NA \\
    Bipedal Walker & 24, 4 dim. & 7 & $314.98 \pm 3.35$ & $\mathbf{315.3 \pm 6.91}$ & \begin{tabular}[c]{@{}l@{}} {\em scalar}: $78.33 \pm 57.19$ (8)\\ {\em linear}: $244.5 \pm 61.84$ (8) \end{tabular} & $301.34 \pm 3.09$ (6) & NA \\
\bottomrule
\end{tabular}
\label{tab:rl}
\end{table*}

\paragraph{Regression Tasks:} We consider the 5 regression datasets from \cite{tao-regression}, plus the two extra from  \cite{dgt2022} (for which TAO-linear results are not available). While TAO-linear \cite{tao-regression} generates DT with regressors at the leaves, similarly as {\em DTSemNet}, allowing them to generalize, DGT \cite{dgt2022} only learns scalar at the leaves, similarly as CART, and thus is less efficient (and cannot generalize), needing deeper DTs to get acceptable accuracy. To understand the impact of regressors at the leaves, we implemented {\em DGT-linear}, a modification of DGT using regressors at the leaves, in the same way as in {\em DTSemNet} and in ICCT \cite{icct2022}. Architectures with regressors at the leaves, namely {\em DTSemNet}(-regression), DGT-linear, and TAO-linear, use the same fixed height following \cite{tao-regression}, while DGT (without regressors) uses deeper trees (as reported in \cite{dgt2022}) and CART is not restricted.
%Table \ref{tab:regression} reports the RMSE results on these regression tasks.}
The results of the regression datasets are presented in Table \ref{tab:regression}. 

{\em DTSemNet}-regression is sometimes second behind TAO-linear (3 benchmarks), and best for the 4 other benchmarks.
{\em DTSemNet} is consistently better than (original) DGT,
by $10\%$ of RMSE on average over the 7 benchmarks, with $50\%$ better
RMSE scores for CTslice, even though DGT uses larger (hard oblique) DTs than {\em DTSemNet}. Compared with  DGT-linear with regressors at the leaves (our adaptation), {\em DTSemNet} is still consistently better or tied, but the advantage is reduced to $3.7\%$ of RMSE on average over the 7 benchmarks, with up to $23\%$ better scores (CTSlice). Two third of the advantage of {\em DTSemNet} over DGT can be attributed to the regressors at the leaves, but still, a meaningful one-third of the advantage can be attributed to reducing the approximations, with only 1 STE call for {\em DTSemNet}-regression instead of $n$ STE calls for DGT and DGT-linear, where $n$ is the height of the DT. Overall, CTSlice is sensitive to approximations.

\medskip
\noindent
\textbf{RL Tasks:} We experimented with four cyber-physical environments with {\em discrete actions} with a limited number of physical features ($\leq 32$):
three environments from OpenAI Gym, namely CartPole (4 features, 2 actions), Acrobot (6 features, 3 actions) and LunarLander (8 features, 4 actions), as well as a larger SC-II environment, namely FindandDestroyZerglings (32 features, 30 actions) \cite{starcraft}. We also experimented with two 
OpenAI environments with continuous actions, namely the continuous version of Lunar Lander (8 features, 2 action dimensions -- horizontal and vertical thruster) and Bipedal Walker (24 features, 4 actions).

We fix the height of the DT that is learned by the different architectures
and compare for reference with Deep RL with an NN with the same number of learned parameters,
which is also used as an expert from which to perform imitation learning for VIPER \cite{bastani2018verifiable}. ICCT and VIPER generate (hard) axis-aligned DT policies, while {\em DTSemNet} and DGT generate (hard) oblique DT policies.
Further, for continuous action space, the regression version of each architecture is used, with 2 variants for DGT: the original with scalars at the leaves and our modification with linear regressions (similar to {\em DTSemNet} and ICCT). Notice that VIPER does not handle continuous action space.

We train 5 policies for each case, using different seeds for the environments. 
We use the standard PPO \cite{schulman2017proximal} for discrete actions and SAC \cite{sac} for continuous actions from the library StableBaseLine3. We run each policy on the same 100 seeds to evaluate the policies and compute the mean reward. Table \ref{tab:rl} reports the mean reward values averaged over the 5 different policies for each architecture.

Overall, {\em DTSemNet} is competitive with Deep RL generating NN-policies on these environments with limited state dimensions ($\leq 32$). Further, {\em DTSemNet} consistently generates the most efficient DT-policies. The more complex the environment, the larger the lead: 
\begin{itemize} 
\item On the simplest CartPole ($N_f=4$), it is tied for first.
\item for Acrobot ($N_f=6$), it leads by a couple of percent.
\item for (discrete action) LunarLander ($N_f=8$), it has a very sizable lead compared to other DT architectures ($\geq 28\%$). Notice that VIPER 
can perform closer ($>200$ reward) to {\em DTSemNet}, but it needs 
very large DTs ($>1300$ leaves) vs 32 leaves for {\em DTSemNet}. 
\item for Zerglings ($N_f=32$), the lead is even more sizable ($\geq 39\%$).
\end{itemize}

Continuous action spaces are easier to handle than the discrete case: 
For Lunar Lander, each architecture with regressors at the leaves has a higher mean reward in the continuous case than in the discrete version, which is expected as it has access to precise continuous actions over 2 dimensions rather than a very coarse choice between 4 discrete actions. Further, rewards are much more similar between the architectures, confirming that the architecture is not as important as in the discrete version since the policies can use a linear actuator with respect to the input state rather than a fixed action as in the discrete version. Even Axis-Aligned DTs come close to {\em DTSemNet}, performing extremely poorly in the discrete version. Again, our linear adaptation for DGT is much more accurate than the original version from \cite{dgt2022}, although the gap with {\em DTSemNet} is still significant for Bipedal Walker.

Regarding the impact of different approximations used by DGT and ICCT, they both use STEs. The \texttt{Sigmoid} (and multiplications) used in ICCT is inefficient for Lunar Lander, particularly for the discrete action variant. 
Quantized gradient descent (DGT due to non-differentiable \texttt{Sign} function) is particularly inefficient in the Zerglings environment and for Bipedal Walker, with poorer results than even the less expressive axis-aligned DTs generated by ICCT. Further, for Bipedal Walker, we observed significant variability in the performance of DGT and ICCT when the tree depth varied, unlike {\em DTSemNet}, which remained consistently $>300$ for depth 6,7,8: the performance of DGT-linear at a depth of 6,7,8 are $112,163,244$, it needs deeper DTs. For ICCT, the performance at depth 6,7,8 is $301,177,112$, severely decreasing with depth, which should not be happening, another proof that approximations hurt the accuracy.

\section{Conclusion}\label{sec:conslusion}
We introduced {\em DTSemNet}, an architecture semantically equivalent to oblique DTs. This architecture enables oblique DTs to be learned using vanilla gradient descent. We demonstrate its performance on supervised classification and regression datasets and RL tasks.
{\em DTSemNet} consistently generates DTs more accurately (or as accurately in easier benchmarks) than competing architecture learning DTs using gradient descent.
This is because {\em DTSemNet}-classification uses no approximation, unlike its competitors, and because {\em DTSemNet}-regression uses fewer approximations.
Further, compared to non-gradient-based methodologies (greedy, non-greedy and global search) for learning DTs, {\em DTSemNet} is significantly faster. {\em DTSemNet}-classification outperforms the best of these methods, reducing the errors by $>10\%$ on harder classification tasks, while the accuracy of {\em DTSemNet}-regression is competitive with the SOTA.

\smallskip
\noindent {\bf Limitations:} {\em DTSemNet} is a DT, making it unsuitable for high-dimensional inputs like images, where DTs struggle with complex shapes and require many leaves, negating their benefits. 

\smallskip
\noindent {\bf Future Work:} For {\em DTSemNet}, the choice of the height of the DT is treated as a hyperparameter, similar to the choice of the number of layers in an NN, in contrast to methods (e.g. CART \citep{breiman1984cart}) that grow trees height. For future work, we will consider developing a regression architecture that does not rely on STE approximations and introduce differentiable methods for tree pruning and adaptive growth.

%%%%%%%%%%%%%%%%%%%%%%%%%%%%%%%%%%%%%%%%%%%%%%%%%%%%%%%%%%%%%%%%%%%%%%%%

%%% Use this environment to include acknowledgements (optional).
%%% This will be omitted in doubleblind mode.
% \newpage
\begin{ack}

This research was conducted as part of the DesCartes program and was supported by the National Research Foundation, Prime Minister’s Office, Singapore, under the Campus for Research Excellence and Technological Enterprise (CREATE) program. This research/project is also supported by the National Research Foundation, Singapore and DSO National Laboratories under the AI Singapore Programme (AISG Award No: AISG2-RP-2020-017). The computational work for this research was partially performed using resources provided by the NSCC, Singapore.

\end{ack}

%%%%%%%%%%%%%%%%%%%%%%%%%%%%%%%%%%%%%%%%%%%%%%%%%%%%%%%%%%%%%%%%%%%%%%%%

%%% Use this command to include your bibliography file.
\bibliography{m2807}

\appendix
\onecolumn
{\centering{\bfseries \LARGE Supplementary Material \\ 
\vskip 20pt}}

\noindent
This supplementary document expands upon the experimental results reported in the main paper, providing additional details about the datasets, train-val-test splits, and hyperparameters. First, in Section \ref{extra_res}, we provide the extended results of the main paper. Then, in Section \ref{dataset}, we discuss the dataset used, its splits, and sources. Finally, in Section \ref{hyperparams}, we detail the hyperparameters used in the experiment.

\section{Extended Results}\label{extra_res}
In Table \ref{tab:perf_height}, the extended results of Table 1 in the main paper are provided, which includes the percentage test accuracy of DT heights 2, 3, and 4 of all discussed DT learning methods. Similarly, for the RL environment Bipedal Walker, the average reward for heights 6, 7, and 8 are presented in Table \ref{tab:walker}. It can be observed that the performance of DGT and ICCT is impacted by changing the height; however, the performance of {\em DTSemNet} remains stable across different heights. 

\begin{table*}[h!]
    \centering
    \caption{Classification percentage accuracy (higher is better) with the most accurate height $\leq 4$, as tested in \cite{ea2023}. For each dataset, we provide the number of features $N_f$, classes $N_c$, and training samples $N_s$. Averaged accuracy $\pm$ std is reported over 100 runs. The performance of height {2,3,4} is reported, and in case the number of leaves exceeds the number of features $N_f$, the results are excluded. We run the {\em DTSemNet} and DGT experiments, while the results for TAO, CART and CRO-DT are copied from \cite{ea2023}.}
    \begin{tabular}{lllccccc}
    \toprule
    Dataset & $N_f$, $N_c$, $N_s$ & Height &{\em DTSemNet} & DGT \cite{dgt2022} & TAO \cite{carreira2018alternating} & CART \cite{breiman1984cart} & CRO-DT \cite{ea2023} \\
    \midrule 
    
    Balance Scale & 4, 3,625 & \begin{tabular}[c]{@{}l@{}} 2 \\ 3 \\ 4 \end{tabular} & \begin{tabular}[c]{@{}l@{}} $\mathbf{90.2 \pm 2.2}$ \\ $89.8 \pm 2.1$ \\ $90.2 \pm 2.2$ \end{tabular} & \begin{tabular}[c]{@{}l@{}} $86.4 \pm 3.7$ \\ $88.0 \pm 1.8$  \\ $88.6 \pm 1.7$ \end{tabular} &
    \begin{tabular}[c]{@{}l@{}} $66.4 \pm 3.6$ \\ $71.7 \pm 2.9$ \\ $77.4 \pm 3.1$ \end{tabular} & 
    \begin{tabular}[c]{@{}l@{}} $64.7 \pm 3.5$ \\ $68.8 \pm 3.2$ \\ $74.9 \pm 3.6$ \end{tabular} &
    \begin{tabular}[c]{@{}l@{}} $69.2 \pm 2.34$ \\ $73.0 \pm 2.6$ \\ $77.8 \pm 3.0$ \end{tabular} \\
    \midrule
    
    Banknote Auth & 4, 4, 1372 & \begin{tabular}[c]{@{}l@{}} 2 \\ 3 \\ 4 \end{tabular} & \begin{tabular}[c]{@{}l@{}} $99.0 \pm 0.5$ \\ $\mathbf{99.8 \pm 0.4}$ \\ $99.8 \pm 0.3$ \end{tabular} & \begin{tabular}[c]{@{}l@{}} $99.5 \pm 0.8$ \\ $\mathbf{99.8 \pm 0.4}$  \\ $99.8 \pm 0.4$ \end{tabular} & \begin{tabular}[c]{@{}l@{}} $89.8 \pm 1.6$ \\ $94.9 \pm 1.4$ \\ $96.6 \pm 1.3$ \end{tabular} & 
    \begin{tabular}[c]{@{}l@{}} $89.4 \pm 1.7$ \\ $91.8 \pm 2.2$ \\ $93.6 \pm 2.2$\end{tabular} &
    \begin{tabular}[c]{@{}l@{}} $90.4 \pm 2.0$ \\ $92.4 \pm 2.3$ \\ $95.2 \pm 2.2$ \end{tabular} \\
    \midrule

    Blood Transfusion & 4, 2,784 & \begin{tabular}[c]{@{}l@{}} 2 \\ 3 \\ 4 \end{tabular} & \begin{tabular}[c]{@{}l@{}} $\mathbf{78.5 \pm 1.7}$ \\ $78.5 \pm 2.0$ \\ $78.5 \pm 2.1$ \end{tabular} & 
    \begin{tabular}[c]{@{}l@{}} $76.4 \pm 2.4$ \\ $77.8 \pm 2.5$  \\ $78.3 \pm 2.4$ \end{tabular} &
    \begin{tabular}[c]{@{}l@{}} $75.6 \pm 1.7$ \\ $76.9 \pm 2.1$ \\ $76.7 \pm 2.2$ \end{tabular} & 
    \begin{tabular}[c]{@{}l@{}} $76.1 \pm 0.74$ \\ $76.5 \pm 1.4$ \\ $77.1 \pm 1.8$\end{tabular} &
    \begin{tabular}[c]{@{}l@{}} $75.7 \pm 1.7$ \\ $76.1 \pm 1.9$ \\ $76.1 \pm 2.2$\end{tabular} \\
    \midrule
    
    Acute Inflam. 1 & 6, 2, 120 & \begin{tabular}[c]{@{}l@{}} 2 \\ 3 \\ 4 \end{tabular} & \begin{tabular}[c]{@{}l@{}} $\mathbf{100 \pm 0}$ \\ $100 \pm 0.3$ \\ $100 \pm 0$ \end{tabular} & 
    \begin{tabular}[c]{@{}l@{}} $98.7 \pm 0.4$ \\ $100 \pm 0.3$  \\ $\mathbf{100 \pm 0}$ \end{tabular} &
    \begin{tabular}[c]{@{}l@{}} $91.2 \pm 5.1$ \\ $99.7 \pm 1.3$ \\ $99.7 \pm 1.2$ \end{tabular} &
    \begin{tabular}[c]{@{}l@{}} $86.7 \pm 8.3$ \\ $\mathbf{100 \pm 0}$ \\ $100 \pm 0$ \end{tabular} &
    \begin{tabular}[c]{@{}l@{}} $\mathbf{100 \pm 0}$ \\ $100 \pm 0$ \\ $100 \pm 0$ \end{tabular} \\
    \midrule

    Acute Inflam. 2 & 6, 2, 120 & \begin{tabular}[c]{@{}l@{}} 2 \\ 3 \\ 4 \end{tabular} & \begin{tabular}[c]{@{}l@{}} $\mathbf{100 \pm 0}$ \\ $100 \pm 0$ \\ $99.9 \pm 0.7$ \end{tabular} & 
    \begin{tabular}[c]{@{}l@{}} $99.9 \pm 0.7$ \\ $99.9 \pm 1.3$  \\ $\mathbf{100 \pm 0}$ \end{tabular} &
    \begin{tabular}[c]{@{}l@{}} $99.0 \pm 2.6$ \\ $99.0 \pm 2.6$ \\ $99.0 \pm 2.6$ \end{tabular} & \begin{tabular}[c]{@{}l@{}} $99.0 \pm 2.6$ \\ $99.0 \pm 2.6$ \\ $99.0 \pm 2.6$ \end{tabular} &
    \begin{tabular}[c]{@{}l@{}} $\mathbf{100 \pm 0}$ \\ $100 \pm 0$ \\ $100 \pm 0$ \end{tabular} \\
    \midrule

    Car Evaluation & 6, 4, 1728 & \begin{tabular}[c]{@{}l@{}} 2 \\ 3 \\ 4 \end{tabular} & \begin{tabular}[c]{@{}l@{}} $82.9 \pm 1.5 $ \\ $88.4 \pm 2.6$ \\ $\mathbf{93.3 \pm 2.2}$ \end{tabular} & 
    \begin{tabular}[c]{@{}l@{}} $80.0 \pm 3.1$ \\ $88.1 \pm 2.3$  \\ $92.1 \pm 2.4$ \end{tabular} &
    \begin{tabular}[c]{@{}l@{}} $77.8 \pm 1.3$ \\ $79.3 \pm 1.5$ \\ $84.5 \pm 1.5$ \end{tabular} & \begin{tabular}[c]{@{}l@{}} $77.8 \pm 1.3$ \\ $77.8 \pm 1.3$ \\ $84.3 \pm 1.4$ \end{tabular} &
    \begin{tabular}[c]{@{}l@{}} $77.6 \pm 1.3$ \\ $80.5 \pm 1.3$ \\ $86.1 \pm 1.3$ \end{tabular} \\
    \midrule

    Breast Cancer & 9, 2, 683 & \begin{tabular}[c]{@{}l@{}} 2 \\ 3 \\ 4 \end{tabular} & \begin{tabular}[c]{@{}l@{}} $\mathbf{97.2 \pm 1.3}$ \\ $97.2 \pm 1.3$ \\ $97.1 \pm 1.3$ \end{tabular} & 
    \begin{tabular}[c]{@{}l@{}} $\mathbf{97.2 \pm 1.2}$  \\ $97.1 \pm 1.3$ \\ $97.2 \pm 1.0$ \end{tabular} &
    \begin{tabular}[c]{@{}l@{}} $94.2 \pm 1.6$ \\ $94.7 \pm 1.6$ \\ $94.7 \pm 1.7$ \end{tabular} &
    \begin{tabular}[c]{@{}l@{}} $93.8 \pm 1.7$ \\ $94.7 \pm 1.7$ \\ $94.7 \pm 1.8$ \end{tabular} &
    \begin{tabular}[c]{@{}l@{}} $95.5 \pm 1.8$ \\ $95.3 \pm 1.7$ \\ $95.2 \pm 1.6$ \end{tabular} \\
    \midrule

    % 53.3 ± 1.0 56.3 ± 1.4 59.7 ± 1.8
    Avila Bible & 10, 12, 10430 & \begin{tabular}[c]{@{}l@{}} 4 \end{tabular} & \begin{tabular}[c]{@{}l@{}} $\mathbf{62.2 \pm 1.4}$ \end{tabular} & 
    \begin{tabular}[c]{@{}l@{}} $59.7 \pm 1.8$ \end{tabular} &
    \begin{tabular}[c]{@{}l@{}}  $55.8 \pm 0.8$ \end{tabular} &
    \begin{tabular}[c]{@{}l@{}}  $54.0 \pm 1.3$ \end{tabular} &
    \begin{tabular}[c]{@{}l@{}}  $59.6 \pm 0.7$ \end{tabular} \\
    \midrule

    % 56.4 ± 1.9 56.4 ± 2.5 56.6 ± 1.4
    Wine Quality Red & 11, 6, 1599 & \begin{tabular}[c]{@{}l@{}} 3 \\ 4 \end{tabular} & \begin{tabular}[c]{@{}l@{}} $\mathbf{58.6 \pm 2.2}$ \\ $57.3 \pm 2.4$ \end{tabular} & 
    \begin{tabular}[c]{@{}l@{}} $56.4 \pm 2.5$ \\ $56.6 \pm 1.4$ \end{tabular} &
    \begin{tabular}[c]{@{}l@{}}  $56.1 \pm 2.3$ \\ $56.9 \pm 2.5$ \end{tabular} &
    \begin{tabular}[c]{@{}l@{}}  $55.3 \pm 2.1$ \\ $55.9 \pm 2.3$ \end{tabular} &
    \begin{tabular}[c]{@{}l@{}}  $55.7 \pm 2.4$ \\ $54.8 \pm 2.9$ \end{tabular} \\
    \midrule

    % 50.5 ± 1.9 51.7 ± 1.8 52.1 ± 1.6
    Wine Quality White & 11, 7, 4898 & \begin{tabular}[c]{@{}l@{}} 3 \\ 4 \end{tabular} & \begin{tabular}[c]{@{}l@{}} $52.6 \pm 1.1$ \\ $\mathbf{53.5 \pm 1.4}$ \end{tabular} & 
    \begin{tabular}[c]{@{}l@{}} $51.7 \pm 1.8$ \\ $52.1 \pm 1.6$ \end{tabular} &
    \begin{tabular}[c]{@{}l@{}}  $52.0 \pm 1.5$ \\ $52.3 \pm 1.4$ \end{tabular} &
    \begin{tabular}[c]{@{}l@{}}  $51.9 \pm 1.3$ \\ $52.0 \pm 1.3$ \end{tabular} &
    \begin{tabular}[c]{@{}l@{}}  $50.9 \pm 1.4$ \\ $50.7 \pm 1.4$ \end{tabular} \\
    \midrule

     % 61.7 ± 2.5 88.4 ± 1.6 89.0 ± 1.6
    Dry Bean & 16, 7, 13611 & \begin{tabular}[c]{@{}l@{}} 3 \\ 4 \end{tabular} & \begin{tabular}[c]{@{}l@{}} $91.2 \pm 0.6$ \\ $\mathbf{91.4 \pm 0.5}$ \end{tabular} & 
    \begin{tabular}[c]{@{}l@{}} $88.4 \pm 1.6$ \\ $89.0 \pm 1.6$ \end{tabular} &
    \begin{tabular}[c]{@{}l@{}}  $78.2 \pm 1.2$ \\ $83.2 \pm 1.5$ \end{tabular} &
    \begin{tabular}[c]{@{}l@{}}  $76.6 \pm 1.4$ \\ $80.5 \pm 1.9$ \end{tabular} &
    \begin{tabular}[c]{@{}l@{}}  $74.3 \pm 3.3$ \\ $77.9 \pm 4.7$ \end{tabular} \\
    \midrule

    Climate Crashes & 18, 2, 540 & \begin{tabular}[c]{@{}l@{}} 2 \\ 3 \\ 4 \end{tabular} & \begin{tabular}[c]{@{}l@{}} $\mathbf{92.9 \pm 1.4}$ \\ $92.1 \pm 1.0$ \\ $91.9 \pm 0.8$ \end{tabular} & 
    \begin{tabular}[c]{@{}l@{}} $92.3 \pm 2.4$ \\ $92.4 \pm 2.6$  \\ $92.3 \pm 2.5$ \end{tabular} &
    \begin{tabular}[c]{@{}l@{}} $90.6 \pm 2.2$ \\ $90.5 \pm 2.3$ \\ $90.1 \pm 2.6$ \end{tabular} &
    \begin{tabular}[c]{@{}l@{}} $91.3 \pm 1.2$ \\ $91.6 \pm 1.5$ \\ $91.8 \pm 1.8$ \end{tabular} &
    \begin{tabular}[c]{@{}l@{}} $91.3 \pm 2.2$ \\ $91.1 \pm 2.0$ \\ $90.7 \pm 2.6$ \end{tabular} \\
    \midrule

    Conn. Sonar & 60, 2, 208 & \begin{tabular}[c]{@{}l@{}} 2 \\ 3 \\ 4 \end{tabular} & \begin{tabular}[c]{@{}l@{}} $79.9 \pm 5.1$ \\ $81.5 \pm 5.0$ \\ $\mathbf{82.1 \pm 5.1}$ \end{tabular} & 
    \begin{tabular}[c]{@{}l@{}} $75.6 \pm 6.5$ \\ $79.5 \pm 5.9$  \\ $80.8 \pm 5.3$ \end{tabular} &
    \begin{tabular}[c]{@{}l@{}} $68.7 \pm 6.0$ \\ $70.9 \pm 6.1$ \\ $70.9 \pm 5.8$ \end{tabular} &
    \begin{tabular}[c]{@{}l@{}} $69.4 \pm 5.8$ \\ $70.1 \pm 6.3$  \\ $70.6 \pm 6.6$ \end{tabular} &
    \begin{tabular}[c]{@{}l@{}} $72.8 \pm 5.9$ \\ $71.4 \pm 5.7$ \\ $72.6 \pm 5.7$ \end{tabular} \\
    \midrule
    
    Optical Recognition & 64, 10, 3823 & \begin{tabular}[c]{@{}l@{}} 4 \end{tabular} & \begin{tabular}[c]{@{}l@{}} $\mathbf{93.3 \pm 1.0}$ \end{tabular} & 
    \begin{tabular}[c]{@{}l@{}} $91.9 \pm 1.0$ \end{tabular} &
    \begin{tabular}[c]{@{}l@{}}  $64.6 \pm 6.5$ \end{tabular} & 
    \begin{tabular}[c]{@{}l@{}}  $53.2 \pm 3.2$ \end{tabular} &
    \begin{tabular}[c]{@{}l@{}}  $65.2 \pm 2.0$ \end{tabular} \\

    \bottomrule
    \end{tabular}
    \label{tab:perf_height}
\end{table*}

\begin{table}[h!]
\caption{Rewards in the Bipedal Walker environment with tree heights of 6, 7, and 8. The number of parameters for the NN is similar to the DT height, with the size of the hidden layer specified in square brackets [] alongside the reward. In the constant regressor, there is only one learnable parameter per leaf, independent of input. Conversely, in linear regression, each leaf represents a linear regression model.}
\label{tab:walker}
\centering
\begin{tabular}{lccccc}
\toprule
Env & Height & {\em DTSemNet} & Deep RL & DGT \cite{dgt2022} & ICCT \cite{icct2022} \\
\midrule
\multirow{5}{*}{Bipedal Walker} & $6$ & $309.7 \pm 4.49$ & $221.3 \pm 40.12$ $[74 \times 74]$ & \begin{tabular}[c]{@{}l@{}} {\em scalar}:  $47.87 \pm 48.29$\\ {\em linear}: $120.91 \pm 76.9$\end{tabular} & $\mathbf{301.34 \pm 3.09}$ \\ \cline{2-6}
 & $7$ & $\mathbf{314.98 \pm 3.35}$ & $297.4 \pm 6.25$ $[110 \times 110]$ & \begin{tabular}[c]{@{}l@{}} {\em scalar}: $32.67 \pm 23.65$\\ {\em linear}: $163 \pm 69.53$\end{tabular} & $176.64 \pm 71.78$ \\ \cline{2-6}
 & $8$ & $306.72 \pm 5.6$ & $\mathbf{315.3 \pm 6.91}$ $[160 \times 160]$ & \begin{tabular}[c]{@{}l@{}} {\em scalar}: $78.33 \pm 57.19$\\ {\em linear}: $\mathbf{244.5 \pm 61.84}$\end{tabular} & $112.24 \pm 60.36$ \\
 \bottomrule
\end{tabular}
\end{table}

\noindent
\textbf{Loss Landscape}: We present the loss landscapes of {\em DTSemNet} and DGT-Linear, similar to those provided in the main paper, using a regression dataset (Ailerons) for demonstration purposes. As depicted in Figure \ref{fig:loss_ls_ailerons}, the loss landscape of {\em DTSemNet} appears slightly flatter compared to that of DGT-Linear. This difference in flatness is less pronounced than what was observed in the loss landscape of MNIST, mainly because the performance metrics are very close in this case — specifically, with an RMSE of $1.66$ for {\em DTSemNet} and $1.67$ for DGT-Linear. Given the minimal performance gap, observing some flatness in the loss landscape for {\em DTSemNet} underscores the impact of architecture on performance.

\begin{figure}[h!]
  \centering
      \includegraphics[scale=0.30]{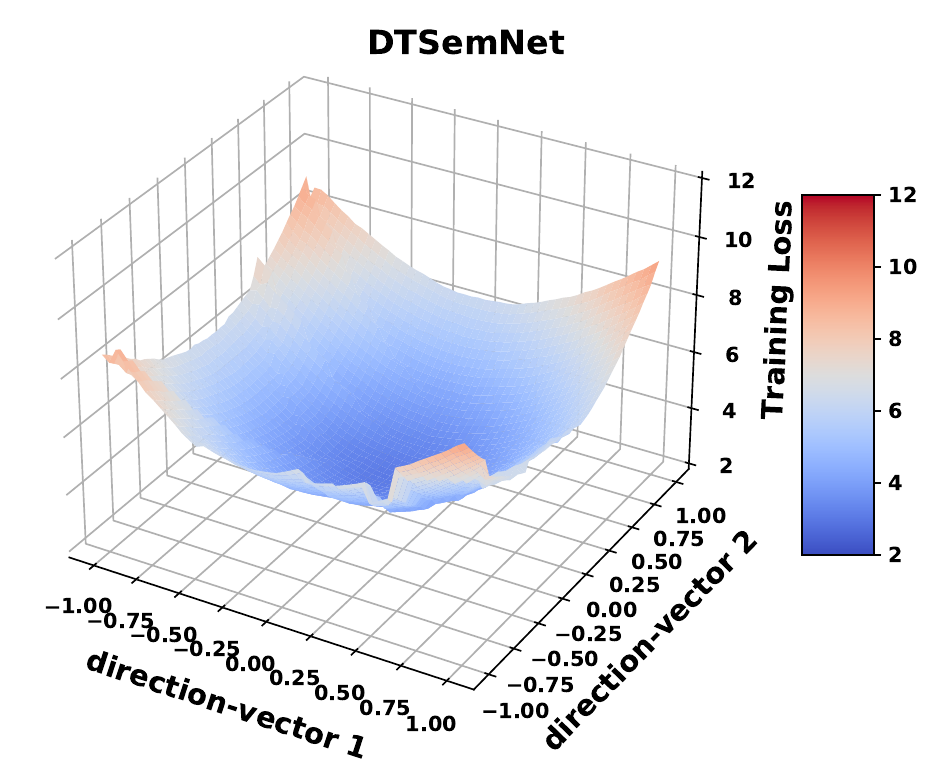}
      \includegraphics[scale=0.30]{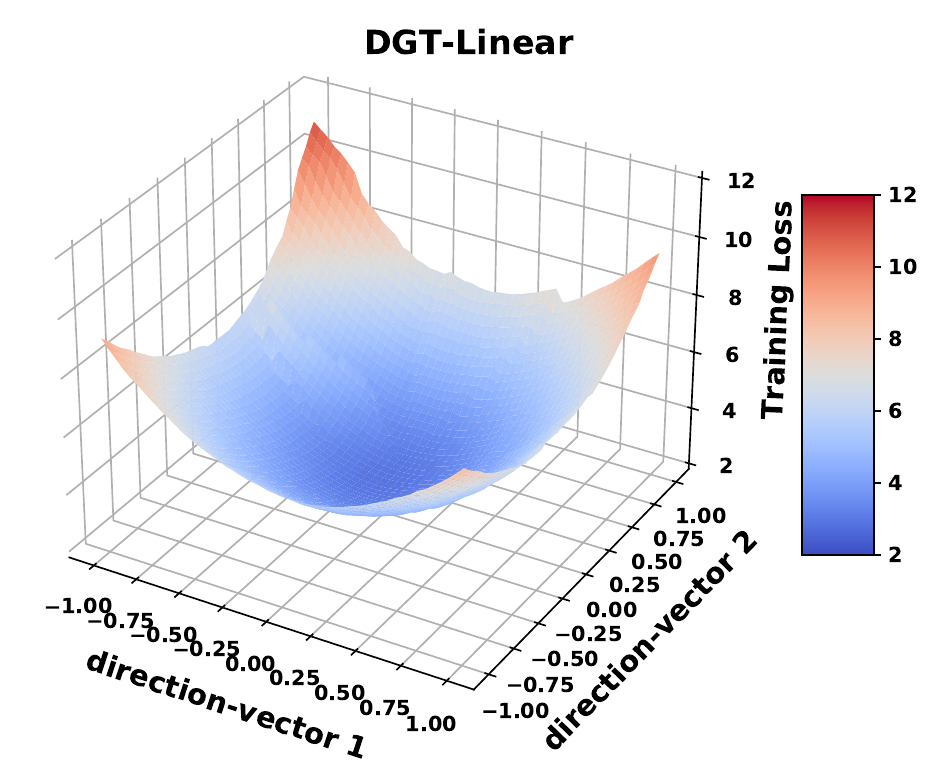}
      \caption{The loss landscape of {\em DTSemNet} and DGT-Linear on  Ailerons regression dataset, when varying the parameters around the trained parameters along two random directions. The flatter the loss landscape, the better the generalization.}
      \label{fig:loss_ls_ailerons}
\end{figure}

\section{Dataset Description}\label{dataset}
This section describes the data used in the experiment, including datasets for classification, regression, and the RL environment. All datasets were preprocessed by standardizing the input features with a mean of 0 and a standard deviation of 1. This normalization procedure was also applied to the target variables in regression datasets. We reuse the explanations of the datasets from \cite{dgt2022}. We used the train-test split provided with the datasets whenever possible, details of which are provided in the tables.

\subsection{Classification}
The 14 classification datasets from \cite{ea2023} shown in Table \ref{tab:cd2} which is sourced from UCI\footnote{\url{https://archive.ics.uci.edu/datasets}}. The 8 classification datasets from \cite{carreira2018alternating} as shown in Table \ref{tab:cd1} is from sourced from LIBSVM\footnote{\url{https://www.csie.ntu.edu.tw/~cjlin/libsvmtools/datasets/}}. 

\begin{table}[!ht]
    \centering
    \caption{Details of classification dataset used in the experiment from UCI.}
    \begin{tabular}{lllll}
    \toprule
        Dataset & train samples & \# features & \# classes & splits (tr:val:test)\\
        \midrule
        Balance Scale & 625 & 4 & 3 & 0.5 : 0.25 : 0.25 \\
        Banknote Auth & 1372 & 4 & 4 & 0.5 : 0.25 : 0.25 \\
        Blood Transfusion & 784 & 4 & 2 & 0.5 : 0.25 : 0.25 \\
        Acute Inflam. 1 & 120 & 6 & 2 & 0.5 : 0.25 : 0.25  \\
        Acute Inflam. 2 & 120 & 6 & 2 & 0.5 : 0.25 : 0.25  \\
        Car Evaluation & 1728 & 6 & 4 & 0.5 : 0.25 : 0.25 \\
        Breast Cancer & 683 & 9 & 2 & 0.5 : 0.25 : 0.25 \\
        Avila Bible & 10430 & 10 & 12 & 0.5 : 0.25 : 0.25 \\
        Wine Quality Red & 1599 & 11 & 6 & 0.5 : 0.25 : 0.25 \\
        Wine Quality White & 4898 & 11 & 7 & 0.5 : 0.25 : 0.25 \\
        Dry Bean & 13611 & 16 & 7 & 0.5 : 0.25 : 0.25  \\
        Climate Crashes & 540 & 18 & 2 & 0.5 : 0.25 : 0.25 \\
        Conn. Sonar & 208 & 60 & 2 & 0.5 : 0.25 : 0.25 \\
        Optical Recognition & 3823 & 64 & 10 & 0.5 : 0.25 : 0.25 \\
        \bottomrule   
    \end{tabular}
    \label{tab:cd2}
\end{table}

\begin{table}[!ht]
    \centering
    \caption{Details of classification dataset used in the experiment from LIBSVM.}
    \begin{tabular}{lllll}
    \toprule
        Dataset & train samples & \# features & \# classes & splits (tr:val:test) \\
        \midrule
        Protein & 14,895 & 357 & 3 & Default \\
        SatImages & 3,104 & 36 & 6 & Default  \\
        Segment & 1,478 & 19 & 7 & 0.64 : 0.16 : 0.2 \\
        PenDigits & 5,995 & 16 & 10 & Default  \\
        Connect4 & 43,236 & 126 & 3 & 0.64 : 0.16 : 0.2  \\  
        MNIST & 48,000 & 780 & 10 & Default \\ 
        SensIT & 63,058 & 100 & 3 & Default \\
        Letter & 10,500 & 16 & 26 & Default \\
         
        \bottomrule   
    \end{tabular}
    \label{tab:cd1}
\end{table}

\begin{itemize}
\item \textbf{Balance Scale}: This data set was generated to model psychological experimental results.  Each example is classified as having the balance scale tip to the right, tip to the left, or be balanced.  The attributes are the left weight, the left distance, the right weight, and the right distance.
\item \textbf{Banknote Auth}: The Banknote Authentication dataset includes features extracted from images of genuine and forged banknotes, with the objective of authenticating banknotes using machine learning techniques.
\item \textbf{Blood Transfusion}: Blood Transfusion dataset contains information about blood donors and their donation history, aiming to predict whether a donor will donate blood again in the future to assist blood donation management.
\item \textbf{Acute Inflam. 1}: The data was created by a medical expert as a data set to test the expert system, which will perform the presumptive diagnosis of two diseases of the urinary system.
\item \textbf{Acute Inflam. 2}: Similar to Acute Inflam. 1, this dataset also deals with cases of acute inflammation, providing additional samples for analysis and classification tasks.
\item \textbf{Car Evaluation}: Car Evaluation dataset involves features of various car models, with the objective of classifying cars into different categories based on their characteristics, such as price, maintenance cost, and safety features.
\item \textbf{Breast Cancer}: The Breast Cancer dataset provides medical data from breast cancer patients, with the goal of accurately classifying tumors as malignant or benign to aid in diagnosis and treatment planning.
\item \textbf{Avila Bible}: The Avila data set has been extracted from 800 images of the 'Avila Bible,' an XII-century giant Latin copy of the Bible. The prediction task consists of associating each pattern with a copyist.
\item \textbf{Wine Quality Red}: Wine Quality Red dataset provides chemical properties of red wine samples, aiming to predict wine quality ratings based on these properties to assist in wine production and quality control.
\item \textbf{Wine Quality White}: The Wine Quality White dataset provides chemical properties of white wine samples, similar to the Wine Quality Red dataset, aiming to predict wine quality ratings based on these properties.
\item \textbf{Dry Bean}: The Dry Bean dataset includes features related to various dry bean varieties, with the objective of predicting bean variety classifications based on these features to assist in agricultural research and crop management.
\item \textbf{Climate Crashes}: Given Latin hypercube samples of 18 climate model input parameter values, predict climate model simulation crashes and determine the parameter value combinations that cause the failures.
\item \textbf{Conn. Sonar}: The Conn. Sonar dataset involves sonar signals bounced off a metal cylinder and a roughly cylindrical rock, aiming to discriminate between these two classes for underwater object detection.
\item \textbf{Optical Recognition}: The Optical Recognition dataset contains handwritten digit images, aiming to recognize and classify digits accurately to facilitate automated postal sorting and document processing tasks.
\end{itemize}

\begin{itemize}
    \item \textbf{Protein}: Given features extracted from amino acid sequences, predict the secondary structure for a protein sequence.
    \item \textbf{Segment}: Given 19 attributes for each 3x3 pixel grid of instances drawn randomly from a database of 7 outdoor color.
    images, predict segmentation of the central pixel.
    \item \textbf{SatImages}: Given multi-spectral values of pixels in 3x3 neighborhoods in a satellite image, predict the central pixel in each neighborhood.
    \item \textbf{PenDigits}: Given integer attributes of pen-based handwritten digits, predict the digit.
    \item \textbf{Connect4}: Given legal positions in a game of connect-4, predict the game theoretical outcome for the first player.
    
    \item \textbf{MNIST}: Given 784 pixels in handwritten digit images, predict the actual digit.
    
    \item \textbf{SensIT}: Given 100 relevant sensor features, classify the vehicles.
    \item \textbf{Letter}: Given 16 attributes obtained from stimulus observed from handwritten letters, classify the actual letters.
\end{itemize}

\iffalse
To stabilize training, the over-parameterization trick is used in the linear layers \cite{dgt2022}. This involves adding more hidden linear layers without any non-linear activations (spurious layers) to increase the number of parameters in the linear layer, which helps accelerate the learning \cite{arora2018optimization}. This is done just during the training and does not affect the actual parameters of the learned DT.

The results in Table \ref{tab:classification1}  and \ref{tab:regression} are averaged over 10 random runs, except for TAO, for which it is averaged over 5 runs. For the results in Table \ref{tab:classification2}, they are averaged over 100 runs. For datasets in Table \ref{tab:regression}, if the train and test data are not explicitly provided, we use five random train-test splits and evaluate each split ten times, as done in \cite{dgt2022}. The supplementary materials provide more details regarding datasets (train size, splits, etc.), train hyperparameters, etc.
\fi

\subsection{Regression Dataset}

We used scalar regression datasets from \cite{dgt2022} shown in Table \ref{tab:reg}. We used the train-test split provided with the datasets whenever possible. When a default split is not provided, we create five different splits and evaluate each of the splits 10 times to obtain the average performance. Then, we average the performance across five splits. Following is the list of datasets:
\begin{itemize}
    
    \item \textbf{Abalone}: Given attributes describing physical measurements, predict the age of an abalone. We encode the categorical (“sex") attribute as one-hot.
    \item \textbf{Comp-Activ}: Given different system measurements, predict the portion of time that CPUs run in user mode.
    \item \textbf{Ailerons}: Given attributes describing the status of the aircraft, predict the command given to its ailerons.
    \item \textbf{CtSlice}: Given attributes as histogram features (in polar space) of the Computer Tomography (CT) slice, predict the relative location of the image on the axial axis (in the range [0 180]).
    \item \textbf{YearPred}: Given several song statistics/metadata (timbre average, timbre covariance, etc.), predict the age of the song. This is a subset of the UCI Million Songs dataset.
    \item \textbf{PdbBind}: Given standard “grid features" (fingerprints of pairs between ligand and protein; see Wu et al. (2018)), predict binding affinities.
    \item \textbf{Microsoft}: Given 136-dimensional feature vectors extracted from query-url pairs, predict the relevance judgment labels, which take values from 0 (irrelevant) to 4 (perfectly relevant).
\end{itemize}

\begin{table}[h]
    \centering
    \caption{Details of regression datasets used from \cite{dgt2022}}
    \begin{tabular}{llllll}
    \toprule
        Dataset & \# features & Train Size & Splits (tr:val:test) & \# shuffles & Source \\ \midrule
        Abalone & 10 & 2,004 & 0.5:0.1:0.4 & 5 & UCI \\ 
        Comp-Activ & 21 & 3,932 & 0.5:0.1:0.4 & 5 & Delve \tablefootnote{\url{https://www.cs.toronto.edu/~delve/data/comp-activ/desc.html}} \\ 
        Ailerons & 40 & 5,723 & Default & 1 & LIACC \tablefootnote{\url{https://www.dcc.fc.up.pt/~ltorgo/Regression/ailerons.html}} \\ 
        CTSlice & 384 & 34,240 & 0.5 : 0.1 : 0.4 & 5 & UCI \\
        YearPred & 90 & 370,972 & Default & 1 & UCI \\
        \midrule
        PDBBind & 2,052 & 9,013 & Default & 1 & MoleculeNet4 \tablefootnote{\url{http://www.pdbbind.org.cn/}} \\ 
        Microsoft & 136 & 578,729 & Default & 1 & MSLR-WEB10K \tablefootnote{\url{https://www.microsoft.com/en-us/research/project/mslr/}} \\       
        \bottomrule
    \end{tabular}
    \label{tab:reg}
\end{table}

\subsection{RL Environments}
In the RL setting 4 discrete environments: CartPole, Acrobot, Lunar Lander and FindAndDefeatZerglings and 2 continuous environments: Continous Lunar Lander and Bipedal walker is used. Following are the details of the environment:

\begin{itemize}
    \item \textbf{CartPole}: Balancing a pole on a cart, the environment has a continuous state space of four dimensions and a discrete action space of two dimensions, with rewards of +1 for each timestep the pole remains upright and 0 upon termination.
    
    \item \textbf{Acrobot}: Involving a two-link pendulum system, the environment features a continuous state space of four dimensions and a discrete action space of three dimensions, with a reward of -1 upon termination.
    
    \item \textbf{LunarLander}: Simulating lunar landing, this environment has a continuous state space of eight dimensions and a discrete action space of four dimensions, with positive rewards for successful landings and negative rewards for crashes.
    
    \item \textbf{FindAndDefeatZerglings}: Tasked with navigating a maze and battling Zerglings, the environment has a discrete state space of 32 dimensions and a discrete action space of 30 dimensions, with positive rewards for defeating enemies and negative rewards for taking damage.
    
    \item \textbf{Continuous Lunar Lander}: Similar to LunarLander but with continuous action space, the environment has a continuous state space of eight dimensions and a continuous action space with two dimensions, with positive rewards for successful landings and crash penalties.
    
    \item \textbf{Bipedal Walker}: Featuring a two-legged robot learning to walk, the environment has a continuous state space of 24 dimensions and a continuous action space of four dimensions, with positive rewards for making progress and penalties for falling or deviating from the desired trajectory.
\end{itemize}

\section{Implementation Details}\label{hyperparams}
In this section, we outline the implementation details for the experiments and provide the sources of the original implementation from which we derived the code.

\medskip
\noindent
\textbf{{\em DTSemNet}}: The implementation of {\em DTSemNet} is done using \texttt{nn.module} class of PyTorch. For the classification and regression datasets, we provide hyperparameters such as optimizer (SGD, RMSProp, or Adam), learning rate (ls), learning momentum (mtm), type of scheduler, scheduler decay parameter, L1 regularizer parameter $\lambda$, batch size, gradient clipping (grad clip) and over-parameterization (overparams) used during training. The over-parameterization is indicated by the size of the used hidden layer for a linear layer. These hyperparameters are tuned using grid search, which is summarized in Tables \ref{tab:hyperparams}. Configurations that are not relevant are available in the configuration files provided in the source code. Furthermore, we provide the hyperparameters used for training agents in the RL setup for six different environments in Tables \ref{para_cart_acr}, \ref{para_ln_zerg} and \ref{para_cont}. These hyperparameters include the choice of optimizer and various parameters for the PPO learner. We obtained these hyperparameters for all agents that use PPO from Stablebaselines through an automated hyperparameter search using the ``Optuna'' library. Optuna performs a Bayesian optimization-based automatic hyperparameter search to find suitable parameters. 

\medskip
\noindent
\textbf{VIPER}: We utilize the online implementation available at \url{https://github.com/Safe-RL-Team/viper-verifiable-rl-impl}. VIPER requires Q-values to train a decision tree. In cases where the policy network is available instead of the Q-network, the authors of VIPER proposed using $\log(\pi)$ for the policy output $\pi$ in place of the Q-value, as it is proportional to it \cite{bastani2018verifiable}.

\medskip
\noindent
\textbf{ICCT}: We use the implementation provided at \url{https://github.com/CORE-Robotics-Lab/ICCT/tree/main} from the authors of ICCT \cite{icct2022} for the architecture of DT. We reuse their code to replace ICCT with {\em DTSemNet} to obtain the results in RL tasks. 

\medskip
\noindent
\textbf{DGT}: We use the implementation available at \url{https://github.com/microsoft/DGT} from the authors of DGT \cite{dgt2022} for our experiments.

\medskip
\noindent
\textbf{CRO-DT}: We use the implementation available at \url{https://github.com/vgarciasc/CRO-DT} from the authors of CRO-DT \cite{ea2023} for the experiments.

\begin{table}[]
\centering
\caption{Hyperparameters used for Experiments with classification and regression datasets. Reported hyperparameters are epoch, optimizer (SGD, RMSProp, or Adam), learning rate (lr), learning momentum (mtm), scheduler type (linear or cosine), scheduler decay parameter, L1 regularizer parameter $\lambda$, batch size, gradient clipping (grad clip) and over-parameterization (overparams).}
\label{tab:my-table}
\begin{tabular}{lllllllllll}
\toprule
Dataset & epoch & optimizer & lr & mtm & \begin{tabular}[c]{@{}l@{}}scheduler\\ type\end{tabular} & \begin{tabular}[c]{@{}l@{}}scheduler\\ decay\end{tabular} & batch size & $\lambda$ & grad clip & overparams \\
\midrule
Segment & 50 & Adam & 0.01 & NA & Linear & 0.95 & 128 & NA & NA & [1530] \\
SatImages & 200 & RMSProp & 0.01 & 0.2 & Linear & 0.98 & 128 & 5e-5 & 0.01 & [4032] \\
PenDigits & 100 & RMSProp & 0.01 & 0.0 & Linear & 0.98 & 128 & 1e-5 & 0.01 & [16320] \\
Letter & 400 & RMSProp & 0.01 & 0.0 & Linear & 0.95 & 128 & 5e-6 & 0.01 & [6128x6138] \\
Protein & 40 & RMSProp & 0.01 & 0.0 & Linear & 0.98 & 128 & 1e-3 & 0.01 & [180x180] \\
Connect4 & 100 & RMSProp & 0.01 & 0.0 & Linear & 0.95 & 128 & 1e-4 & 0.01 & [8160] \\
MNIST & 100 & SGD & 0.4 & 0.9 & Linear & 0.95 & 128 & NA & NA & NA \\
SensIT & 250 & RMSProp & 0.01 & 0.0 & Linear & 0.95 & 128 & 1e-4 & 0.01 &[6138] \\
\midrule

Acute Inflam. 1 & 20 & Adam & 0.8 & NA & Linear & 0.99 & 128 & NA & NA & NA \\
Acute Inflam. 2 & 30 & Adam & 0.7 & NA & Linear & 0.98 & 128 & NA & NA & NA \\
Conn. Solar & 40 & Adam & 0.1 & NA & Linear & 0.95 & 128 & 1e-3 & NA & NA \\
Climate Crashes & 30 & Adam & 0.4 & NA & Linear & 0.90 & 128 & 1e-2 & NA & NA \\
Balance Scale & 20 & Adam & 0.8 & NA & Linear & 0.98 & 128 & NA & NA & NA \\
Breast Cancer & 50 & Adam & 0.05 & NA & Linear & 0.90 & 256 & 1e-2 & NA & NA \\
Blood Transfusion & 50 & Adam & 0.7 & NA & Linear & 0.95 & 128 & NA & NA & NA \\
Banknote Auth & 40 & Adam & 0.5 & NA & Linear & 0.98 & 128 & NA & NA & NA \\
Wine Quality Red & 60 & Adam & 0.4 & NA & Linear & 0.98 & 128 & NA & NA & NA \\
Car Evaluation & 40 & Adam & 0.8 & NA & Linear & 0.99 & 128 & NA & NA & NA \\
Optical Recognition & 100 & Adam & 0.8 & NA & Linear & 0.95 & 128 & 1e-3 & NA & NA \\
Wine Quality White & 60 & Adam & 0.1 & NA & Linear & 0.99 & 128 & NA & NA & NA \\
Avila Bible & 100 & Adam & 0.3 & NA & Linear & 0.98 & 128 & NA & NA & NA \\
Dry Bean & 40 & Adam & 0.1 & NA & Linear & 0.99 & 128 & 1e-3 & NA & NA \\

\midrule
Ailerons & 100 & Adam & 0.1 & NA & Linear & 0.90 & 32 & 5e-4 & 0.01 & NA \\
Abalone & 50 & Adam & 0.005 & NA & Linear & 0.95 & 32 & 5e-4 & NA & [248] \\
Comp-Activ & 100 & Adam & 0.02 & NA & Linear & 0.90 & 128 & 1e-4 & NA & [186] \\
CTSlice & 50 & Adam & 0.001 & NA & Linear & 0.90 & 128 & NA & NA & [992x496] \\
PDBBind & 100 & Adam & 0.01 & NA & Linear & 0.95 & 256 & 1e-2 & NA & NA \\
YearPred & 40 & Adam & 0.001 & NA & Linear & 0.90 & 128 & 1e-4 & NA & [378] \\
Microsoft & 30 & Adam & 0.002 & NA & Linear & 0.90 & 256 & NA & NA & NA \\
\bottomrule
\end{tabular}
\label{tab:hyperparams}
\end{table}

\begin{table}[h]
  \centering
  \caption{Hyperparameters of agents in CartPole (left) and Acrobot (right) Environment.}
  \begin{minipage}{0.49\textwidth}
    \resizebox{\linewidth}{!}{
      \begin{tabular}{lccccc}
        \hline
        Parameters & {\em DTSemNet} & NN & DGT & ICCT  \\
        \hline
        Size (height/nodes)   & 4  & 16x16 & 4 & 4   \\
        Learning Rate & 0.024 & 0.015 & 0.0042 & 0.024 \\
        Batch Size & 32 & 32 & 43 & 32  \\
        Num Steps & 832 & 512 & 542 & 832 \\
        Entropy Coeff & 0.02 & 0.008  & 0.02 & 0.02 \\
        Clip Range & 0.2 & 0.2 & 0.2 & 0.2 \\
        VF Coeff & 0.5 & 0.5 & 0.5 \\
        GAE Lamda & 0.86 & 0.96  & 0.934 & 0.86 \\
        Max Grad Norm & 0.5 & 0.5  & 0.53 & 0.5 \\
        Num Epochs & 20 & 12 &  7 & 20 \\
        Critic Network & \multicolumn{4}{c}{ --------- \quad 16x16 \quad---------} \\
        Optimizer & \multicolumn{4}{c}{ --------- \quad Adam \quad ---------} \\
        Discount Factor & \multicolumn{4}{c}{ --------- \quad 0.99 \quad ---------} \\
        \hline
        \end{tabular}
        % \caption{Hyperparameters of agents in Cart Pole Environment.}
        % \label{para_cart}
    }
  \end{minipage}
  \hfill
  \begin{minipage}{0.49\textwidth}
    \resizebox{\linewidth}{!}{
      \begin{tabular}{lccccc}
        \hline
        Parameters & {\em DTSemNet} & NN & DGT & ICCT   \\
        \hline
        Size (height/nodes)   & 4  & 16x16  & 4 & 4    \\
        Learning Rate & 0.045 & 0.007 & 0.001 & 0.045  \\
        Batch Size & 64 & 64 & 128 & 64  \\
        Num Steps & 1024 & 1024 & 2048 & 1024  \\
        Entropy Coeff & 0.01 & 0.02  & 0.012 & 0.01  \\
        Clip Range & 0.2 & 0.2 & 0.2 & 0.2  \\
        VF Coeff & 0.5 & 0.5 & 0.54 & 0.5  \\
        GAE Lamda & 0.80 & 0.85  & 0.94 & 0.80 \\
        Max Grad Norm & 0.52 & 0.5 & 0.5 & 0.52 \\
        Num Epochs & 16 & 16 & 20 & 16  \\
        Critic Network & \multicolumn{4}{c}{ --------- \quad 32x32 \quad---------} \\
        Optimizer & \multicolumn{4}{c}{ --------- \quad Adam \quad ---------} \\
        Discount Factor & \multicolumn{4}{c}{ --------- \quad 0.99 \quad ---------} \\
        \hline
        \end{tabular}

    }
  \end{minipage}
  
  \label{para_cart_acr}
\end{table}

\begin{table}[h]
\caption{Hyperparameters of agents in Lunar Lander (left) and FindAndDestroy (right) Environment.}
  \centering
  \begin{minipage}{0.49\textwidth}
    \resizebox{\linewidth}{!}{
        \begin{tabular}{lccccc}
        \hline
        Parameters & {\em DTSemNet} & NN & DGT & ICCT  \\
        \hline 
        Size (height/nodes)  & 5 & 32x32 & 5 & 5    \\
        Learning Rate & 0.02 & 0.005 & 0.02 & 0.005 \\
        Batch Size & \multicolumn{4}{c}{ --------- \quad 64 \quad---------}  \\
        Num Steps & 2048 & 2048 & 2048 & 2048 \\
        Entropy Coeff & \multicolumn{4}{c}{ --------- \quad 0.01 \quad---------} \\
        Clip Range & \multicolumn{4}{c}{ --------- \quad 0.2 \quad---------} \\
        VF Coeff & \multicolumn{4}{c}{ --------- \quad 0.5 \quad---------} \\
        GAE Lamda & 0.9 & 0.9  & 0.9 & 0.9  \\
        Max Grad Norm & \multicolumn{4}{c}{ --------- \quad 0.5 \quad---------}\\
        Num Epochs & 20 & 10 & 4 & 32 \\
        Critic Network & \multicolumn{4}{c}{ --------- \quad 32x32 \quad---------} \\
        Optimizer & Adam & Adam & Adam & Adam \\
        Discount Factor & \multicolumn{4}{c}{ --------- \quad 0.99 \quad ---------} \\
        \hline
        \end{tabular}
    }
  \end{minipage}
  \hfill
  \begin{minipage}{0.49\textwidth}
    \resizebox{\linewidth}{!}{
        \begin{tabular}{lccccc}
        \hline
        Parameters & {\em DTSemNet} & NN & DGT & ICCT  \\
        \hline 
        Size (height/nodes)  & 6 & 64x64 & 6 & 6 \\
        Learning Rate & 0.02 & 0.005 & 0.02 & 0.005 \\
        Batch Size & \multicolumn{4}{c}{ --------- \quad 64 \quad---------}  \\
        Num Steps & 2048 & 2048 & 2048 & 2048 \\
        Entropy Coeff & \multicolumn{4}{c}{ --------- \quad 0.01 \quad---------} \\
        Clip Range & \multicolumn{4}{c}{ --------- \quad 0.2 \quad---------} \\
        VF Coeff & \multicolumn{4}{c}{ --------- \quad 0.5 \quad---------} \\
        GAE Lamda & 0.9 & 0.9  & 0.9 & 0.9  \\
        Max Grad Norm & \multicolumn{4}{c}{ --------- \quad 0.5 \quad---------}\\
        Num Epochs & 10 & 10 & 4 & 4 \\
        Critic Network & \multicolumn{4}{c}{ --------- \quad 128x128 \quad---------} \\
        Optimizer & \multicolumn{4}{c}{ --------- \quad Adam \quad ---------} \\
        Discount Factor & \multicolumn{4}{c}{ --------- \quad 0.99 \quad ---------} \\
        \hline
        \end{tabular}
    }
  \end{minipage}
  
  \label{para_ln_zerg}
\end{table}

\begin{table}[h]
\caption{Hyperparameters of agents in ContinousLunarLander Environment (left) and Bipedal Walker (right).}
  \centering
  \begin{minipage}{0.48\textwidth}
    \resizebox{\linewidth}{!}{
        \begin{tabular}{lccccc}
\hline
Parameters & {\em DTSemNet} & NN & DGT & ICCT  \\
\hline 
Size (height/nodes)  & 4 & 16x16 & 4 & 4   \\
Learning Rate & 1e-3 & 1e-4  & 1e-3 & 1e-4 \\
Batch Size & \multicolumn{4}{c}{ --------- \quad 256 \quad---------}  \\
Tau & \multicolumn{4}{c}{ --------- \quad 0.01 \quad---------}  \\
Train Freq & \multicolumn{4}{c}{ --------- \quad 1 \quad---------}  \\
Grad Steps & \multicolumn{4}{c}{ --------- \quad 1 \quad---------}  \\
Critic Network & \multicolumn{4}{c}{ --------- \quad 256x256 \quad---------} \\
Optimizer & \multicolumn{4}{c}{ --------- \quad Adam \quad---------} \\
Discount Factor & \multicolumn{4}{c}{ --------- \quad 0.99 \quad ---------} \\
\hline
\end{tabular}
    }
  \end{minipage}
  \hfill
  \begin{minipage}{0.5\textwidth}
    \resizebox{\linewidth}{!}{
        \begin{tabular}{lccccc}
\hline
Parameters & {\em DTSemNet} & NN & DGT & ICCT  \\
\hline 
Size (height/nodes)  & 7 & 160x160 & 8 & 6    \\
Learning Rate & 2e-3 & 5e-4  & 2e-3 & 1e-3 \\
Batch Size & \multicolumn{4}{c}{ --------- \quad 256 \quad---------}  \\
Tau & \multicolumn{4}{c}{ --------- \quad 0.02 \quad---------}  \\
Train Freq & \multicolumn{4}{c}{ --------- \quad 64 \quad---------}  \\
Grad Steps & \multicolumn{4}{c}{ --------- \quad 64 \quad---------}  \\
Critic Network & \multicolumn{4}{c}{ --------- \quad 512x512 \quad---------} \\
Optimizer & \multicolumn{4}{c}{ --------- \quad Adam \quad---------} \\
Discount Factor & \multicolumn{4}{c}{ --------- \quad 0.98 \quad ---------} \\
\hline
\end{tabular}
    }
  \end{minipage}
  
  \label{para_cont}
\end{table}

\end{document}